\documentclass[11pt,article]{article}

\usepackage[margin=1in]{geometry}
\usepackage[utf8]{inputenc}
\usepackage[T1]{fontenc}
\usepackage{lmodern}
\usepackage{microtype}
\usepackage{amsmath,amssymb,amsthm,mathtools,bm}
\usepackage{xcolor}
\usepackage{hyperref}
\usepackage{enumitem}
\usepackage{comment}

\usepackage{tikz}
\usetikzlibrary{positioning,arrows.meta}
\usetikzlibrary{calc}

\usepackage{algorithm}
\usepackage{algpseudocode}

\renewcommand{\div}{\mathrm{div} }

\newcommand{\supp }{\operatorname{supp}}

\newcommand{\R}{\mathbb{R}}

\newcommand{\KSD}{\mathrm{KSD}}

\newcommand{\nn}{\nonumber}

\newcommand{\KL}{\mathsf{KL}}

\newcommand{\brac}[1]{\left( #1\right)}
\allowdisplaybreaks

\theoremstyle{plain}
\newtheorem{theorem}{Theorem}[section]
\newtheorem{lemma}[theorem]{Lemma}

\newtheorem{proposition}[theorem]{Proposition}

\theoremstyle{definition}

\newtheorem{remark}[theorem]{Remark}
\newtheorem{assumption}[theorem]{Assumption}

\newcommand{\blue}[1]{\textcolor{blue}{#1}} 
\newcommand*{\QEDA}{\hfill\ensuremath{\square}}

\definecolor{deeporange}{RGB}{180,80,0}

\hypersetup{
  colorlinks=true,
  linkcolor=blue,
  urlcolor=teal,
  citecolor=magenta,
  pdftitle={}
}

\title{Anchored Stein Variational Gradient Descent Algorithm for Non-smooth Sampling}

	\date{October 26, 2026}

\author{ 
    	Thanh Dang  \thanks{Department of Computer Science, University of Rochester, Rochester, NY 14620. 
        (email: {\tt ycloud777@gmail.com}).}
        \qquad
        Matthew X. Burns \thanks{Department of Electrical and Computer Engineering, University of Rochester, Rochester, NY 14627 (email:
{\tt mburns13@ur.rochester.edu}).}\qquad
	 	Jiaming Liang \thanks{
        Goergen Institute for Data Science and Artificial Intelligence and Department of Computer Science, University of Rochester, Rochester, NY 14620 (email: {\tt jiaming.liang@rochester.edu}).}
	}
\begin{document}
\maketitle

\begin{abstract}
We introduce a Stein variational framework for sampling from a target density $p\propto e^{-U}$ when the target potential $U$ is possibly non-smooth. The key idea is to introduce a smoother reference potential $U_0$ and the anchor weight $a=e^{U-U_0}$. The resulting anchored Stein variational gradient descent (SVGD) algorithm is able to target $p\propto e^{-U}$ without computing the gradient of $U$ and include standard SVGD as a special case. 

In particular, we allow the target potential $U$ to be locally Lipschitz while the reference potential $U_0$ is continuously differentiable, along with other suitable assumptions. Under this low regularity regime,  we establish a Kullback-Leibler first variation formula for anchored perturbations, identify the corresponding steepest descent direction in the reproducing kernel Hilbert space and prove a nonasymptotic convergence bound for the anchored SVGD algorithm in terms of an anchored kernelized Stein discrepancy. Finally, we also provide guidance on constructing a suitable reference potential $U_0$ based on $U$. 
\end{abstract}

\section{Introduction and main result}

Stein Variational Gradient Descent (SVGD) \cite{LiuWang} is a particle
method for variational inference. Its updates follow the direction of steepest descent of the Kullback-Leibler (KL) divergence
among directions in the unit ball of a reproducing kernel Hilbert space (RKHS). In \cite{liu2017flow}, Liu established the gradient flow interpretation of
SVGD, which can be viewed as a discretized and kernelized counterpart of
Wasserstein gradient flow for KL divergence.  Kernelization makes this formulation computationally
tractable: the particle updates depend only on the target score, kernel
values and kernel derivatives, without requiring score estimation that is usually present when discretizing Wasserstein gradient flow. Moreover, this also enables the connection of SVGD and kernel methods. The latter has proved to be very useful for SVGD by 
allowing Hessian information to be incorporated through
matrix-valued kernels for preconditioning \cite{wang2019matrix}, and
kernel interactions to be evaluated on lower dimensional projections
for sampling in high dimensions \cite{gong2021sliced,liu2022grassmann}, among other examples.

Over the last decade, many algorithmic extensions of SVGD have
been developed for different sampling problems. Mirrored methods address
constrained domains \cite{shi2022mirrored}, Riemannian SVGD accommodates
manifold geometry \cite{liu2018riemannian}, and branching SVGD promotes
exploration of multimodal distributions \cite{banales2025branching}.
Stein variational Newton methods \cite{detommaso2018newton} and SVGD with
matrix valued kernels \cite{wang2019matrix} incorporate curvature and
preconditioning into the particle updates. Projected
\cite{chen2020projected}, sliced \cite{gong2021sliced}, and Grassmann
\cite{liu2022grassmann} variants use dimension reduction to address
sampling in high dimensions, while Stein variational message passing
exploits conditional dependencies in graphical models
\cite{wang2018message}. Other extensions reduce the cost of particle
interactions through random batches \cite{li2020stochastic}, train neural
networks to generate samples \cite{feng2017amortized}, construct adaptive
importance proposals \cite{han2017importance}, remove the need to
tune a learning rate \cite{sharrock2023coin}, and enable inference for
discrete distributions \cite{han2020discrete}.

We now present our anchored SVGD algorithm for the purpose of sampling from targets
with nonsmooth potentials.
\begin{algorithm}[H]
\caption{Anchored SVGD}
\label{alg:anch:svgd}
\begin{algorithmic}[1]
\Require Target $p\propto e^{-U}$ where $U$ is locally Lipschitz, anchor potential $U_0$,
         anchor weight $a(x)=e^{U(x)-U_0(x)}$, kernel $k$, initial particles $\{x_i^0\}_{i=1}^N$,
         stepsizes $\{\varepsilon_\ell\}_{\ell=0}^{L-1}$.

\For{$\ell=0,1,\dots,L-1$}

  \State Define 
  \[
    \widehat\phi_\ell(x)
    := \frac1N\sum_{j=1}^N a(x_j^\ell)\Big(\nabla_{x_j} k(x_j^\ell,x)
       - k(x_j^\ell,x) \nabla U_0(x_j^\ell)\Big),
  \]
  where $\nabla_{x_j}k(\cdot,x)$ is the gradient w.r.t.\ the first argument.

  \State Update all particles:
  \begin{align}
  \label{def:x:ell:seq}
    x_i^{\ell+1}   \leftarrow   x_i^\ell + \varepsilon_\ell  a(x_i^\ell) \widehat\phi_\ell(x_i^\ell),
    \qquad i=1,\dots,N.
  \end{align}

\EndFor

\State \Return $\{x_i^L\}_{i=1}^N$.
\end{algorithmic}
\end{algorithm}
Algorithm~\ref{alg:anch:svgd} is motivated by the anchored Langevin dynamics of G\"urb\"uzbalaban et al. \cite{gurbuzbalaban2025anchored}, which introduce a reference potential $U_0$, the anchor weight $a=e^{U-U_0}$, and the dynamics
\begin{align}
\label{eq:anch:lang}
dX_t=-\nabla U_0(X_t)e^{U(X_t)-U_0(X_t)}dt+\sqrt{2} e^{(U(X_t)-U_0(X_t))/2}dW_t,
\end{align}

preserving the target law $p\propto e^{-U}$. Their framework is designed to handle targets for which the original potential $U$ may be non-differentiable, by working instead with the more regular reference potential $U_0$.  Our goal here is to construct a deterministic interacting particle analogue of this anchored sampling mechanism in the Stein variational setting. This leads naturally to Algorithm~\ref{alg:anch:svgd}: the factor $a=e^{U-U_0}$ is inherited from the anchored Langevin drift in \eqref{eq:anch:lang}, while the stochastic diffusion term is discarded and the particle evolution is selected by the anchored KL steepest descent principle. In our analysis of anchored SVGD, we work under the regularity assumptions that $U$ is locally Lipschitz and $U_0$ is continuously differentiable.

At the level of the explicit particle update, Algorithm~\ref{alg:anch:svgd} is close to the gradient-free SVGD update of Han and Liu in \cite{HanLiu} after identifying their target density with $p\propto e^{-U}$ and their auxiliary density with $p_0\propto e^{-U_0}$. However, our assumptions and subsequent analysis are different from theirs. In \cite{HanLiu}, the target and auxiliary densities are both assumed to be positive and differentiable, and the main step therein is to rewrite the algorithm update as standard SVGD with an importance-weighted kernel $w(x)w(y)k(x,y)$ for $w(x)=\frac{\rho(x)}{p(x)}$, so that the standard SVGD theory immediately applies after reweighting, including the key KL first variation formula that is Theorem 3.1 of Liu and Wang in \cite{LiuWang}. In the present paper, by contrast, we work in an anchored nonsmooth setting where the target potential $U$ is only locally Lipschitz, while there exists a reference potential $U_0$ which is continuously differentiable. An important part of our work is establishing a  KL first variation formula under these low regularity assumptions in Theorem~\ref{theo:bl:ext}. The first variation formula then allows us to identify in Theorem \ref{theo:opt:dir} the steepest descent direction field in Algorithm~\ref{alg:anch:svgd}. 
\begin{theorem}
\label{theo:bl:ext} Suppose Assumptions~\ref{ass:phi:bl} and~\ref{ass:bl:ext} hold. For a bounded Lipschitz function $\phi$, define the anchored velocity field $\psi:=a\phi$. Then the flow
\begin{align*}
\dot\Phi_t=\psi\circ\Phi_t,\qquad  \Phi_0=\mathrm{Id}
\end{align*}
is well defined, and the pushforward density $q_t := (\Phi_t)_\# q$ satisfies the continuity equation
\begin{align*}
  \partial_t q_t + \div(q_t\psi)=0,\qquad \psi=a\phi.
\end{align*}
Moreover, define
\begin{align*}
  F(t):=\KL(q_t\|p)
  = \int_{\R^d} q_t(x)\log\frac{q_t(x)}{p(x)} dx.
\end{align*}
Then for every $t_0>0$,  $F$ is finite and absolutely continuous on $[-t_0,t_0]$. In particular, the KL first variation at $t=0$ is
\begin{align*}
  \left. \frac {d}{dt}F(t)\right|_{t=0}= -\mathbb E_q  \Big[a(x)\big(\div\phi(x)-\nabla U_0(x)\cdot\phi(x)\big)\Big].
\end{align*}
\end{theorem}


We will write $\mathcal H^d$ as the $d$-fold product RKHS with inner product
$\langle \phi,\psi\rangle_{\mathcal H^d}:=\sum_{i=1}^d \langle \phi_i,\psi_i\rangle_{\mathcal H}$. Based on Theorem~\ref{theo:bl:ext}, we can identify in the next theorem the steepest descent direction for the KL divergence within the anchored RKHS class $\psi=a\phi$, $\phi\in\mathcal H^d$.
\begin{theorem}\label{theo:opt:dir} 
Assume Assumptions~\ref{ass:bl:ext} and~\ref{ass:rkhs}. For $\phi\in\mathcal H^d$, define the anchored velocity field $\psi:=a\phi$. Then the flow
\begin{align}
\label{def:fl:repr:sp}
\dot\Phi_t=\psi\circ\Phi_t,\qquad  \Phi_0=\mathrm{Id}
\end{align}
is well defined, and we write $q_t := (\Phi_t)_\# q$. 

Furthermore, let $\mathbb B_{\mathcal H^d}:=\{\phi\in\mathcal H^d:\|\phi\|_{\mathcal H^d}\le 1\}$ and define
\[
\phi^\star:=
\begin{cases}
h/\|h\|_{\mathcal H^d},&h\neq0,\\
0,&h=0,
\end{cases}
\]
where for each $i$
\begin{align}
\label{def:opt:dir:comp}
     h_i(\cdot)
  = \mathbb E_{X\sim q}\!\Big[a(X)\Big(\partial_{x_i}k(X,\cdot)-(\partial_i U_0)(X) k(X,\cdot)\Big)\Big]
  \in\mathcal H. 
\end{align}
Then for every $\phi\in\mathcal H^d$,
\[
D\mathrm{KL}(q)[a\phi]
:=
\left.\frac{d}{dt}\mathrm{KL}(q_t\|p)\right|_{t=0}
=-\langle h,\phi\rangle_{\mathcal H^d}.
\]
Moreover, we have the De Bruijn-type identity
\[
  \left.\frac{d}{dt}\mathrm{KL}(q_t\|p)\right|_{t=0,\ \phi=\phi^\star}=-\sup_{\phi\in\mathbb B_{\mathcal H^d}}\mathbb E_q\!\big[a(\mathrm{div} \phi-\nabla U_0\cdot \phi)\big] = -\|h\|_{\mathcal H^d}.
\]

\end{theorem}

\begin{remark}
The transport equation
\[
\dot\Phi_t=\psi\circ\Phi_t,\qquad \psi:=a\phi,\qquad a=e^{U-U_0},
\]
is motivated by the anchored Langevin dynamics \eqref{eq:anch:lang}, whose drift is $-a\nabla U_0$. In the deterministic Stein variational setting, it is natural to mimic the drift of the anchored Langevin dynamics and perturb the law along anchored direction fields $a\phi$ rather than along the field $\phi$. 
\end{remark}



We now turn to quantitative convergence of Algorithm~\ref{alg:anch:svgd}. For each iteration $\ell$, let $g(\ell)$ denote the law of the particle vector $x^\ell=(x_1^\ell,\dots,x_N^\ell)\in(\R^d)^N$ and let
\begin{equation}
\label{def:disc:emp:meas}
\mu_\ell^N:=\frac1N\sum_{i=1}^N \delta_{x_i^\ell}
\end{equation}
be the associated empirical measure. The theorem below gives a nonasymptotic bound on the time-averaged anchored KSD of the empirical measures $\mu_\ell^N$ along the discrete anchored SVGD trajectory. The anchored KSD, introduced in Section~\ref{sec:anch:ksd}, contains the standard KSD as the special case when $U= U_0$.

\begin{theorem}
\label{theo:disc:bd:ksd2}
   Let Assumptions~\ref{ass:rkhs} and~\ref{ass:stab} hold with parameters $B,\alpha_0$ and $C_{U_0}$ defined therein. Furthermore, 
\begin{enumerate}
    \item $a=e^{U-U_0}$ is bounded and Lipschitz, with
\[
A_\infty:=\|a\|_{L^\infty},
\qquad
L_a:=\|\nabla a\|_{L^\infty},
\]
\item There exists $C_\Delta>0$ such that
\begin{equation}
\label{eq:disc:lap:u0:bd}
\sup_{x\in\mathbb R^d} |\Delta U_0(x)| \le C_\Delta d.
\end{equation}
\item Let $x^0=(x_1^0,\dots,x_N^0)\in(\mathbb R^d)^N$
be the initial particle configuration in Algorithm~\ref{alg:anch:svgd}, and let $g(0)$ denote its joint density on $(\mathbb R^d)^N$. Let $K>0$ and define
\begin{equation}
\label{def:gk:set}
G_K:=\left\{x\in(\mathbb R^d)^N:\ \frac1N\sum_{i=1}^N U_0(x_i)\le K\right\},
\end{equation}
Assume that
\begin{align}
\label{eq:mu0:gk:cond}
\mathrm{KL}\bigl(g(0)\|p_0^{\otimes N}\bigr)\le \gamma dN,
\qquad
\mu_0(G_K):=\int_{G_K} g(0,x) dx\ge \frac12.
\end{align}
Set the truncation of the initial density $g(0)$ to the set $G_K$ as
\begin{align}\label{def:gk:trunc:dens}g_K(0,x):=
\frac{g(0,x)\mathbf 1_{G_K}(x)}{\mu_0(G_K)}\end{align}
\item Let
\begin{equation}
\label{eq:disc:param:defs}
\beta:=\frac{2\alpha_0}{1-\alpha_0},
\qquad
\Gamma_K:=M^{2\alpha_0}\Bigl(d^{\frac{1}{1-\alpha_0}}+K\Bigr)^{2\alpha_0},
\qquad
C_a:=4L_a^2A_\infty^2+5A_\infty^4,
\end{equation}
where $M,D$ are the constants from Lemma~\ref{lem:disc:bd:sum:u0}(ii), and define
\begin{align}
\label{eq:disc:c0anc}
C_0
&:=
\max\!\Bigl\{
1, 
D, 
4\bigl[
2C_{U_0}A_\infty^4A_0^2B^2\Gamma_K
+
2C_{U_0}A_\infty^4B^2d
\notag\\
&\hspace{3.8em}
+
4A_0^2B^2C_a d \Gamma_K
+
4B^2(C_a+3A_\infty^4C_{U_0}^2)d^2
\bigr]
\Bigr\}.
\end{align}
\item Choose
\begin{equation}
\label{eq:disc:unnor:t:theta}
\overline L:=N^{\frac{2-\alpha_0}{1-2\alpha_0}},
\qquad
L:=\lceil\overline L\rceil,
\qquad
\theta=\sqrt{\frac{N}{L}},
\end{equation}
and set
\begin{equation}
\label{eq:disc:unnor:eta}
\varepsilon_\ell \equiv \eta,
\qquad
\eta
=
\left[
\frac{1}{C_0}
\left(
\frac{1}{N^{1-\alpha_0}L^{2\alpha_0}}
\wedge
\frac1N
\right)
\right]^{1/2}\theta.
\end{equation}
\end{enumerate}

Then with respect to the anchored KSD at \eqref{eq:anch:ksd},
\begin{equation}
\label{eq:disc:fin:anch}
\mathbb E_{g_K(0)}
\Biggl[
\frac1L\sum_{\ell=0}^{L-1}\KSD_a^2\bigl(\mu_\ell^N\|p\bigr)
\Biggr]
\le
\frac{1}{N}\brac{
\left(2\gamma d+1+\frac2e+\log2\right)\sqrt{C_0}
+\frac{1}{L}\sum_{\ell=0}^{L-1}
\Bigl(
R_3^{\mathrm{disc}}(\ell)
+R_{\mathrm{diag}}^{\mathrm{disc}}(\ell)
\Bigr)}.
\end{equation}

In particular,
\begin{align}
\label{eq:disc:r3bd}
|R_3^{\mathrm{disc}}(\ell)|
&\le
N\beta_{3}^{(1)} d
+
N\beta_{3}^{(2)} A_0 M^{\alpha_0}
\Bigl(d^{\frac{1}{1-\alpha_0}}+K\Bigr)^{\alpha_0}
\Bigl((\eta \ell)^{\frac{\alpha_0}{1-\alpha_0}}\vee 1\Bigr),\\
\label{eq:disc:rdiag:bd}
|R_{\mathrm{diag}}^{\mathrm{disc}}(\ell)|
&\le
\beta_{\mathrm{diag}}^{(1)} d
+
\beta_{\mathrm{diag}}^{(2)} A_0 M^{\alpha_0}
 d^{1/2}
\Bigl(d^{\frac{1}{1-\alpha_0}}+K\Bigr)^{\alpha_0}
\Bigl((\eta \ell)^{\frac{\alpha_0}{1-\alpha_0}}\vee 1\Bigr),
\end{align}
where
\begin{equation}
\label{eq:disc:beta3:beta:d}
\beta_{3}^{(1)}=\beta_{3}^{(2)}:=A_\infty L_a B,
\quad
\beta_{\mathrm{diag}}^{(1)}:=
B\bigl(A_\infty L_a+A_\infty^2+A_\infty^2 C_\Delta\bigr),
\quad
\beta_{\mathrm{diag}}^{(2)}:=
B\bigl(A_\infty L_a+A_\infty^2\bigr).
\end{equation}
\end{theorem}

The paper is organized as follows. Section~\ref{sec:prelim}
introduces notations, the anchored kernelized Stein discrepancy and a construction
of the reference potential $U_0$ from the target potential $U$ by convolution smoothing. We defer
the proofs of these results to Appendices~\ref{appen:anch:ksd}
and~\ref{appen:ref:pot} to keep the main text focused on the algorithm
and its analysis. Section~\ref{sec:firstvariation} establishes the
KL first variation formula and identifies the direction of steepest
descent in the anchored RKHS class. We first develop the properties
of Lipschitz flows and their transported densities needed to justify
differentiating the KL divergence when $U$ is only locally Lipschitz.
Section~\ref{sec:anch:svgd} proves the convergence bound for the
anchored SVGD algorithm through a discrete descent estimate for the
KL divergence relative to the reference density $p_0\propto e^{-U_0}$. Supporting
bounds on the particle updates and the growth of $U_0$ along the
iterations for the results in Section~\ref{sec:anch:svgd} are collected in Appendix~\ref{sec:disc:supp:lem}.

\section{Preliminaries}
\label{sec:prelim}

\subsection{Notations and function spaces}
We collect here notations and function spaces used throughout.

\begin{itemize}[leftmargin=2em]

  \item \emph{(Function spaces).}
  For $1\le p\le\infty$,
  \[
    L^p(\R^d)
    :=\bigl\{f:\R^d\to\R\ \text{measurable}:\ \|f\|_{L^p}<\infty\bigr\},
    \quad
    \|f\|_{L^p}
    :=
    \begin{cases}
      \bigl(\int_{\R^d}|f|^p\bigr)^{1/p}, & 1\le p<\infty,\\[0.2em]
      \operatorname*{ess sup}_{x\in\R^d}|f(x)|, & p=\infty.
    \end{cases}
  \]
  We write $L^p_{\mathrm{loc}}(\R^d)$ for functions that belong to
  $L^p(K)$ for every compact $K\subset\R^d$.
  For vector-valued maps we use the same notation with
  $\|\psi\|_{L^p}:=\||\psi|\|_{L^p}$ and
  $|\psi(x)|=(\sum_{i=1}^d|\psi_i(x)|^2)^{1/2}$.

  For an open $\Omega\subset\R^d$ and $1\le p\le\infty$, the Sobolev space
  \[
    W^{1,p}(\Omega)
    :=\{f\in L^p(\Omega):\partial_i f\in L^p(\Omega),\ i=1,\dots,d\}
  \]
  is equipped with
  $\|f\|_{W^{1,p}(\Omega)}
    :=\|f\|_{L^p(\Omega)}+\|\nabla f\|_{L^p(\Omega)}$.
  We use $W^{1,p}_{\mathrm{loc}}(\R^d)$ for local Sobolev regularity.
  For vector fields
  $W^{1,p}(\Omega;\R^d):=\{\psi=(\psi_1,\dots,\psi_d):\psi_i\in W^{1,p}(\Omega)\}$.
  In particular we frequently use $W^{1,1}(\R^d)$,
  $W^{1,\infty}(\R^d)$ and $W^{1,\infty}(\R^d;\R^d)$.

  \item \emph{(Distributional solutions).}
  We write $\mathcal D(\Omega):=C_c^\infty(\Omega)$
  for smooth functions with compact support in an open
  $\Omega\subset\R^d$ or $\Omega\subset\R^{1+d}$, and
  $\mathcal D'(\Omega)$ for the space of distributions.  The duality pairing is written
  $\langle T,\varphi\rangle$ for $T\in\mathcal D'(\Omega)$,
  $\varphi\in\mathcal D(\Omega)$.
  Vector-valued test functions are denoted
  $C_c^\infty(\R^d;\R^d)$.

  When we write
  \[
    \partial_t q_t + \div(q_t\psi)=0
    \quad\text{in }\mathcal D'((0,t_0)\times\R^d),
  \]
  this means that $q_t$ satisfies the continuity equation in the sense
  of distributions: for every $\varphi\in C_c^\infty((0,t_0)\times\R^d)$,
\[
  \int_0^{t_0}\int_{\R^d} q_t(x)\bigl(\partial_t\varphi(t,x)
  +\psi(x)\cdot\nabla\varphi(t,x)\bigr) dx dt = 0.
\]

  \item \emph{(Lipschitz maps).}
  For a map $f:\R^d\to\R^m$,
  \[
    \mathrm{Lip}(f)
    :=\sup_{x\neq y}\frac{|f(x)-f(y)|}{|x-y|}
    \in[0,\infty]
  \]
  is the (global) Lipschitz constant; we say $f$ is globally Lipschitz if
  $\mathrm{Lip}(f)<\infty$ and locally Lipschitz,
  $f\in\mathrm{Lip}_{\mathrm{loc}}(\R^d)$, if $f$ is Lipschitz on every compact set.
  A map $\phi:\R^d\to\R^d$ is bounded–Lipschitz if
  \[
    \|\phi\|_{\mathrm{BL}}
    :=\|\phi\|_{L^\infty(\R^d)}+\mathrm{Lip}(\phi)<\infty,
  \]
  and we write
  $\mathrm{BL}(\R^d;\R^d)$ for the space of such vector fields.
  Any $\psi\in W^{1,\infty}(\R^d;\R^d)$ is bounded and globally
  Lipschitz, hence belongs to $\mathrm{BL}(\R^d;\R^d)$.

  \item \emph{(Euclidean Balls).}
  For $x_0\in\R^d$ and $R>0$,
  \[
    B_R(x_0):=\{x\in\R^d:|x-x_0|\le R\},
    \qquad
    B_R:=B_R(0).
  \]
  We use $|\cdot|$ for the Euclidean norm on $\R^d$, and $\div$,
  $\nabla$ for the usual divergence and gradient operators.

  \item \emph{(Flows, Jacobians, pushforwards).}
  Given a (globally) Lipschitz vector field
  $\psi:\R^d\to\R^d$, we write $(\Phi_t)_{t\in\R}$ for the (unique)
  global flow solving the ODE
  \[
    \dot\Phi_t(x)=\psi(\Phi_t(x)),\qquad \Phi_0(x)=x.
  \]
  For each fixed $t$, $\Phi_t:\R^d\to\R^d$ is a bi-Lipschitz
  homeomorphism (an invertible function where both the map and its inverse are Lipschitz continuous) and almost everywhere differentiable; its Jacobian
  determinant is denoted $J_t(x):=\det D\Phi_t(x)$.

  If $q$ is a (Borel) probability density on $\R^d$, the pushforward density of
  $q$ by $\Phi_t$ is $q_t=(\Phi_t)_\# q$ defined by
  \[
    \int_{\R^d} f(x) q_t(x) dx
    = \int_{\R^d} f(\Phi_t(x)) q(x) dx
    \quad\text{for all bounded Borel }f.
  \]

  \item \emph{(Expectations and KL divergence).}
  For an integrable function $f:\R^d\to\R$ and a density $q$ we set
  \[
    \mathbb E_q[f]:=\int_{\R^d} f(x) q(x) dx.
  \]
  Given strictly positive densities $q,p$ with $q\log(q/p)\in L^1(\R^d)$,
  the Kullback–Leibler divergence is
  \[
    \KL(q\|p)
    :=\int_{\R^d} q(x)\log\frac{q(x)}{p(x)} dx.
  \]
  In the anchored setting we write $p=Z^{-1}e^{-U}$ for a potential
  $U:\R^d\to\R$ and normalizing constant $Z\in(0,\infty)$; when
  convenient we suppress $Z$ in intermediate formulas.

  \item \emph{(Bochner spaces).}
  For a Banach space $X$ and $t_0>0$ we write
  \[
    C([0,t_0];X)
    :=\{t\mapsto f_t\in X:\ t\mapsto f_t
          \text{ is continuous in the norm of }X\},
  \]
  and similarly $C^1([0,t_0];X)$ for continuously differentiable
  paths. Bochner spaces $L^p((0,t_0);X)$, $1\le p\le\infty$, consist of strongly
measurable maps $t\mapsto f_t\in X$ with finite
$\|f_\cdot\|_{L^p((0,t_0);X)}$. Here, strong measurability means that there
exist measurable simple maps
$s_n(t)=\sum_{j=1}^{N_n}\mathbf 1_{E_{n,j}}(t)x_{n,j}$, with measurable
sets $E_{n,j}\subset(0,t_0)$ and $x_{n,j}\in X$, such that
$\|s_n(t)-f_t\|_X\to0$ for a.e.\ $t$. In the rest of the paper, we mainly use
$X=L^1(\R^d)$ and $X=L^\infty(\R^d)$.

  \item \emph{(Reproducing kernel Hilbert spaces).}
Let $k:\R^d\times\R^d\to\R$ be a symmetric positive-definite kernel. The reproducing kernel Hilbert space (RKHS) associated with $k$ is the unique Hilbert space $\mathcal H$ of real-valued functions on $\R^d$ such that $k(x,\cdot)\in\mathcal H$ for every $x\in\R^d$, and the reproducing property
\[
  f(x)=\langle f,k(x,\cdot)\rangle_{\mathcal H}
  \qquad\forall f\in\mathcal H,\ \forall x\in\R^d
\]
holds. For vector-valued fields we write
\[
  \mathcal H^d:=\mathcal H\times\cdots\times\mathcal H,
\]
equipped with the product inner product
\[
  \langle \phi,\psi\rangle_{\mathcal H^d}
  :=\sum_{i=1}^d \langle \phi_i,\psi_i\rangle_{\mathcal H},
  \qquad
  \|\phi\|_{\mathcal H^d}^2=\sum_{i=1}^d \|\phi_i\|_{\mathcal H}^2.
\]
\end{itemize}


\subsection{Anchored kernelized Stein discrepancy}
\label{sec:anch:ksd}

Let $k:\R^d\times\R^d\to\R$ be a symmetric positive-definite kernel, $\mathcal H$ its RKHS, and $\mathcal H^d$ the associated $d$-fold RKHS. Suppose we have a target potential $U$ and a reference potential $U_0$. The target density is $p\propto e^{-U}$ and the reference density is $p_0\propto e^{-U_0}$. The Langevin Stein operator associated with $p_0$ is $A_{p_0}^{\top}\phi:=\div\phi-\nabla U_0\cdot\phi$. 
For a probability measure $q$ on $\mathbb R^d$, the anchored kernelized Stein discrepancy (KSD) of $q$ with respect to $p$ is defined by
\begin{equation}
\label{eq:anch:ksd:orig}
\mathrm{KSD}_a(q\|p)
:=
\sup\Bigl\{
\mathbb E_q\bigl[a(X)A_{p_0}^{\top}\phi(X)\bigr]:
\ \phi\in\mathcal H^d,\ \|\phi\|_{\mathcal H^d}\le 1
\Bigr\},
\end{equation}
We observe that when $U\equiv U_0$, we have $a\equiv 1$ and $p_0=p\propto e^{-U}$, so \eqref{eq:anch:ksd:orig} reduces to the standard KSD \cite{LLJ,OGC,CSG}.

The following assumptions will be used for the upcoming result. 
\begin{assumption}
\label{ass:q0:kern}
Let $q$ be a strictly positive probability density on $\mathbb R^d$, and define
\begin{equation}
\label{def:qzero}
q_0(x):=\frac{a(x)q(x)}{\int_{\mathbb R^d} a(y)q(y) dy}.
\end{equation}
Assume:
\begin{enumerate}[label=(Q\arabic*),leftmargin=2.6em]
\item\label{q:c0:univ}
The scalar kernel $k$ is bounded and $C_0$-universal on $\mathbb R^d$.
Equivalently, the RKHS $\mathcal H$ is dense in $C_0(\mathbb R^d)
:=
\Bigl\{f\in C(\mathbb R^d): f(x)\to 0 \text{ as } |x|\to\infty\Bigr\}$
with respect to the uniform norm.

\item\label{q:integ}
The quantities $a(X)\nabla\log p_0(X), a(X)\nabla\log q_0(X)$ and $ a(X)\nabla\log\frac{p(X)}{q(X)}$ belong to $L^1(q;\mathbb R^d)$.

\item\label{q:rat:loc:lip}
The function $r(x):=\log\frac{p(x)}{q(x)}$
belongs to $\mathrm{Lip}_{\mathrm{loc}}(\mathbb R^d)$.
\end{enumerate}
\end{assumption}

\medskip
\textbf{Assumption M.2.}
For every $x\in\mathbb R^d$ and every $j\in\{1,\dots,d\}$,
\begin{equation}
\label{eq:q0:kern:stein}
\mathbb E_{Y\sim q_0}\!\Big[
(\partial_j\log q_0)(Y) k(x,Y)+\partial_{y_j}k(x,Y)
\Big]=0.
\end{equation}

\smallskip
The following assumption is a direct sufficient condition for
Assumption~\blue{M.2} (see the proof of Lemma~\ref{lem:q0xi0} for the
implication \blue{M.2'} $\Rightarrow$ \blue{M.2}).

\textbf{Assumption M.2'.}
The density $q_0$ belongs to
$W_{\mathrm{loc}}^{1,1}(\mathbb R^d)$ and, for every
$x\in\mathbb R^d$,
\[
q_0( \cdot )k(x, \cdot )\in W^{1,1}(\mathbb R^d).
\]

We derive an explicit formula for anchored KSD and present a zero detection result for anchored KSD, motivated by \cite[Theorem 2.2]{CSG}. The proofs and supporting results are presented in Appendix~\ref{appen:anch:ksd}.
\begin{theorem}
\label{theo:anch:ksd:form:zero}
Assume Assumptions~\ref{ass:rkhs} and~\ref{ass:bl:ext}(A1), (A2), and (A4)(a). Then:
\begin{enumerate}[label=(\roman*),leftmargin=2.6em]
\item
\begin{align}
\label{eq:anch:ksd}
\KSD_a^2(q\|p)
=
\mathbb E_{(X,Y)\sim q\otimes q}\big[u_a(X,Y)\big],
\end{align}
where
\begin{align}
\label{def:u:a}
u_a(x,y):=a(x)a(y)\Big(
\nabla U_0(x)\cdot(k(x,y)\nabla U_0(y))
-\nabla U_0(x)\cdot\nabla_2 k(x,y)\nn\\
\qquad\qquad
-\nabla U_0(y)\cdot\nabla_1 k(x,y)
+\nabla_1\cdot\nabla_2 k(x,y)
\Big).
\end{align}

\item
Assume in addition Assumption~\ref{ass:q0:kern}, and either Assumption~\blue{M.2} or Assumption~\blue{M.2'}. Then 
\[
\mathrm{KSD}_a(q\|p)=0 \quad\Longrightarrow\quad q=p
\quad\text{a.e.}. 
\]
\end{enumerate}
\end{theorem}

\subsection{Constructing reference potential $U_0$} 
\label{sec:ref:pot}

In the following result, we provide a general way to construct the reference potential $U_0$ from the target potential $U$ in order to satisfy our assumptions. This result is motivated by \cite[Section 5.2]{gurbuzbalaban2025anchored} which proposes Gaussian convolution smoothing to construct $U_0$. In this paper, we consider smoothing via general densities, which include the Gaussian density as a special case. We also relax the Lipschitzness assumption on $U$ in \cite[Section 5.2]{gurbuzbalaban2025anchored}.

The proof of the following result is in Appendix~\ref{appen:ref:pot}.
\begin{lemma}
\label{lem:u0:moll}
Let $U:\mathbb R^d\to\mathbb R$ be a locally Lipschitz function such that
\[
0<\int_{\mathbb R^d}e^{-U(x)} dx<\infty
\]
and
\begin{equation}
\label{eq:u:sec:der}
|\nabla U(x)-\nabla U(y)|\le L_1|x-y|+L_2
\qquad\text{for a.e. }(x,y)\in\mathbb R^d\times\mathbb R^d,
\end{equation}
for some constants $L_1,L_2>0$.

Let $\rho:\mathbb R^d\to[0,\infty)$ satisfy $\rho\in C^2(\mathbb R^d),
\int_{\mathbb R^d}\rho(z) dz=1,\int_{\mathbb R^d} z \rho(z) dz=0$ and for $\alpha=(\alpha_1,\dots,\alpha_d)\in\mathbb N_0^d$ with $|\alpha|\le 2$,
\begin{align}
\label{eq:const:u0:mom:cond}
\int_{\mathbb R^d} (1+|z|^2) |\partial^\alpha \rho(z)| dz<\infty. 
\end{align}
Define the smoothing
\begin{align}
\label{def:u0:moll}
U_0(x):=(U*\rho)(x)
=\int_{\mathbb R^d} U(x-z)\rho(z) dz.
\end{align}
Recall $a(x):=\exp\bigl(U(x)-U_0(x)\bigr)$. Then $(U,U_0,a)$ satisfies Assumption~\ref{ass:bl:ext}(A1) and (A2).

Moreover, if we further assume that
\begin{align}
\label{eq:uolem:add:ass}
\inf_{x\in\mathbb R^d} U(x)>0,
\qquad
|\nabla U(x)|\le A_0 U(x)^{\alpha_0}
\quad\text{for a.e. }x\in\mathbb R^d,
\end{align}
for some constants $A_0>0$ and $\alpha_0\in[0,1/2)$, then $U_0$ satisfies Assumption~\ref{ass:stab}(S1) with the same $\alpha_0,A_0$, and with $C_{U_0}$ depending only on $d,L_1,L_2,\rho$.
\end{lemma}


\section{KL first variation formula and optimal direction field}
\label{sec:firstvariation}

\subsection{Proof overview for the KL first variation formula}

Consider the transport equation \begin{align}\label{eq:trans:cls}\dot\Phi_t=\psi\circ\Phi_t,\qquad \psi=a\phi\end{align} and assume it is well-defined. If the initial density is $q$, then the family of pushforward densities is \begin{align}\label{eq:push:sk} q_t=(\Phi_t)_{\#}q, t\geq 0.\end{align} Let $p=Z^{-1}e^{-U}$ be another density and define the function 
\[
 F(t):=\KL(q_t\|p).
\]
Let us first provide a \textit{formal, non-rigorous} derivation of $F'(0)$. If all objects were sufficiently smooth and integrable, then the continuity equation
\begin{align}
\label{eq:cont:eq:cls}
  \partial_t q_t+\operatorname{div}(q_t\psi)=0
\end{align}
has a unique classical solution. Combining with Leibniz's integral rule and integration by parts to get
\begin{align}
\label{eq:kl:first:var:cls}
  \frac{d}{dt}F(t)=
  \int_{\mathbb R^d}(\log q_t+1+U)\partial_t q_t dx  &\stackrel{\eqref{eq:cont:eq:cls}}{=}
  -\int_{\mathbb R^d}(\log q_t+1+U)\operatorname{div}(q_t\psi) dx \notag\\&=
  \int_{\mathbb R^d}
  \nabla(\log q_t+U)\cdot \psi q_t dx .
\end{align}
At $t=0$, we have  $ F'(0)=\int_{\mathbb R^d}
  \bigl(\nabla\log q+\nabla U\bigr)\cdot a\phi q dx $. Then after applying $\int_{\mathbb R^d} a\phi\cdot\nabla q dx
  =
  -\int_{\mathbb R^d} q \operatorname{div}(a\phi) dx$ (due to integration by parts) and recalling that $ a=e^{U-U_0}$, we arrive at 
\begin{align}\label{intro:kl:first:var:0} F'(0)=
  -\mathbb E_q\bigl[
  a(\operatorname{div}\phi-\nabla U_0\cdot\phi)
  \bigr],\end{align}
which is the desired KL first variation formula in Theorem \ref{theo:bl:ext}.

We emphasize that the above argument to obtain $F'(0)$ is only formal and not rigorous, and a key contribution of this paper is to derive \eqref{intro:kl:first:var:0} rigorously under
low regularity on $U$ in Assumption~\ref{ass:bl:ext}, namely $U$ is only locally Lipschitz. Under Assumption~\ref{ass:bl:ext}, $q_t$ is the pushforward of $q$ by a Lipschitz flow, so that the continuity equation \eqref{eq:cont:eq:cls}
is only well-defined in the distributional sense. Thus, the substitution of \eqref{eq:cont:eq:cls} in \eqref{eq:kl:first:var:cls} is not justified.

We summarize next a rigorous approach to derive
\eqref{intro:kl:first:var:0}. Let $\Phi_t$ be the flow generated by the
bounded Lipschitz vector field $\psi=a\phi$, and let
$J_t=\det\nabla\Phi_t$. The pushforward formula gives
\begin{align}
\label{eq:kl:trans:sk}
F(t)-F(0)
=
\int_{\R^d}q(x)
\Bigl[U(\Phi_t(x))-U(x)-\log J_t(x)\Bigr]dx.
\end{align}
The Liouville identity gives
$\log J_t(x)=\int_0^t(\div\psi)(\Phi_r(x))dr$, while the chain rule gives
$U(\Phi_t(x))-U(\Phi_s(x))
=\int_s^t\nabla U(\Phi_r(x))\cdot\psi(\Phi_r(x))dr$ for a.e.\ $x$.
Fubini's theorem and $q_r=(\Phi_r)_\#q$ then give
\[
F(t)-F(s)
=
\int_s^t\int_{\R^d}q_r(x)
\bigl(\nabla U(x)\cdot\psi(x)-\div\psi(x)\bigr)dxdr.
\]
Since $\psi=a\phi$,
$\div(a\phi)=\nabla a\cdot\phi+a\div\phi$, and
$\nabla a=a(\nabla U-\nabla U_0)$ a.e., this becomes
\[
F(t)-F(s)
=
\int_s^t\int_{\R^d}q_r(x)g(x)dxdr,
\qquad
g(x):=-a(x)\bigl(\div\phi(x)-\nabla U_0(x)\cdot\phi(x)\bigr).
\]

We then show that
$r\mapsto\int_{\R^d}q_r(x)g(x)dx$ is continuous on $I$.
Hence \eqref{eq:kl:der:int:form} leads to $F\in C^1(I)$ and also \eqref{intro:kl:first:var:0}.

\subsection{General transport results}

\begin{lemma} 
\label{lem:push:timed}
Let $I\subset\mathbb R$ be a compact interval containing $0$, and let $b\in L^\infty\!\bigl(I;W^{1,\infty}(\mathbb R^d;\mathbb R^d)\bigr)$ be strongly measurable. Set
\begin{align}
\label{def:bound:const}
M_I:=\operatorname*{ess sup}_{t\in I}\|b_t\|_{L^\infty},
\qquad
L_I:=\operatorname*{ess sup}_{t\in I}\|\nabla b_t\|_{L^\infty},
\qquad
D_I:=\operatorname*{ess sup}_{t\in I}\|\div b_t\|_{L^\infty}.
\end{align}
Then there exists a unique two-parameter flow $(\Phi_{t,s})_{s,t\in I}$ satisfying
\begin{equation}\label{eq:flow:time:dep}
\Phi_{t,s}(x)
=x+\int_s^t b_r(\Phi_{r,s}(x)) dr. 
\end{equation}
Furthermore, it has the following properties.
\begin{enumerate}[label=(\roman*),leftmargin=2em]
\item
$\Phi_{t,r}\circ\Phi_{r,s}=\Phi_{t,s}$,
$\Phi_{s,s}=\mathrm{Id}$, and
$\Phi_{t,s}^{-1}=\Phi_{s,t}$.

\item
For every $s,t\in I$ and $x,y\in\mathbb R^d$,
\begin{align*}
&|\Phi_{t,s}(x)-x|\le M_I|t-s|\quad\text{and}\quad e^{-L_I|t-s|}|x-y|
\le|\Phi_{t,s}(x)-\Phi_{t,s}(y)|
\le e^{L_I|t-s|}|x-y|. 
\end{align*}
In particular, $\Phi_{t,s}:\mathbb R^d\to\mathbb R^d$ is a bi-Lipschitz homeomorphism.

\item
There is a positive Borel function $J_{t,s}(x)$ which, for each $s,t$,
coincides a.e.\ with $\det\nabla\Phi_{t,s}(x)$. For every fixed $s$, there
is a null set $N_s$ such that, for $x\notin N_s$, the map
$t\mapsto J_{t,s}(x)$ is absolutely continuous and
\begin{equation}\label{eq:liouv:lip}
\partial_tJ_{t,s}(x)
=(\div b_t)(\Phi_{t,s}(x))J_{t,s}(x),
\qquad J_{s,s}(x)=1,
\end{equation}
for a.e.\ $t\in I$. Moreover, $\log J_{t,s}(x)
=\int_s^t(\div b_r)(\Phi_{r,s}(x)) dr$ for every $t\in I$ and
\begin{equation}\label{eq:bd:jt}
e^{-D_I|t-s|}\le J_{t,s}(x)\le e^{D_I|t-s|}
\quad\text{for a.e.\ }x.
\end{equation}

\item
For every nonnegative measurable
$\varphi:\mathbb R^d\to[0,\infty]$ and every $s,t\in I$, we have the
change of variables
\begin{equation}\label{eq:area:form}
\int_{\mathbb R^d}\varphi(\Phi_{t,s}(x)) dx
=
\int_{\mathbb R^d}
\frac{\varphi(y)}
{J_{t,s}(\Phi_{s,t}(y))} dy.
\end{equation}
Furthermore,
\begin{equation}\label{eq:jac:inv}
\frac{1}{J_{t,s}(\Phi_{s,t}(y))}
=J_{s,t}(y)
\qquad\text{for a.e.\ }y.
\end{equation}

\item
If $q\in L^1(\mathbb R^d)$ and $q_{t,s}$ is defined by
\[
(\Phi_{t,s})_\#(q(x)dx)=q_{t,s}(y)dy,
\]
then
\begin{equation}\label{eq:push:expl:gen}
q_{t,s}(y)
=
\frac{q(\Phi_{s,t}(y))}
{J_{t,s}(\Phi_{s,t}(y))}
=q(\Phi_{s,t}(y))J_{s,t}(y)
\qquad\text{for a.e.\ }y.
\end{equation}
In particular,
\begin{equation}\label{eq:push:expl}
q_{t,s}(\Phi_{t,s}(x))
=\frac{q(x)}{J_{t,s}(x)}
\qquad\text{for a.e.\ }x.
\end{equation}

\item
If $q\in L^1(\mathbb R^d)$ and $q_{t,s}$ is defined as in Part v, then
for every $s,t\in I$,
\begin{equation}\label{eq:pres:l1}
\|q_{t,s}\|_{L^1}
=\|q\|_{L^1}.
\end{equation}
In particular, if $q$ is a probability density, then so is $q_{t,s}$.

\item
If $q\in L^1(\mathbb R^d)$, write
$\Phi_t:=\Phi_{t,0}$ and $q_t:=q_{t,0}$. Then the map
$t\mapsto q_t$ is continuous from $I$ to $L^1(\mathbb R^d)$.

\item
If $q\in L^1(\mathbb R^d)$, write
$\Phi_t:=\Phi_{t,0}$ and $q_t:=q_{t,0}$. Then
\[
\partial_tq_t+\nabla\cdot(q_tb_t)=0
\quad\text{in the weak sense}.
\]

\end{enumerate}
\end{lemma}

\begin{proof}

\medskip
\noindent\underline{Part i:} \cite[Lemma 8.1.4]{ambrosioGigliSavare2005} gives, for every
$s\in I$ and $x\in\mathbb R^d$, a unique absolutely continuous
solution of \eqref{eq:flow:time:dep}. For every $r,s,t\in I$, uniqueness gives $\Phi_{t,r}\bigl(\Phi_{r,s}(x)\bigr)=\Phi_{t,s}(x)$ so that  $\Phi_{s,t}\circ\Phi_{t,s}=\operatorname{Id}, \Phi_{t,s}\circ\Phi_{s,t}=\operatorname{Id}$ and $\Phi_{t,s}^{-1}=\Phi_{s,t}$. 

\medskip
\noindent\underline{Part ii:} The displacement bound follows from
\[
|\Phi_{t,s}(x)-x|
=
\left|\int_s^t b_r(\Phi_{r,s}(x)) dr\right|
\le \int_0^{|t-s|}M_I dr
=M_I|t-s|.
\]
The upper bound on $|\Phi_{t,s}(x)-\Phi_{t,s}(y)|$ is a consequence of the Lipschitz bound for $b_r$
and Gr\"onwall's inequality applied to
$|\Phi_{t,s}(x)-\Phi_{t,s}(y)|$. Applying this estimate to the inverse
$\Phi_{s,t}$ gives the lower bound and the claimed bi-Lipschitz property.

\medskip
\noindent\underline{Part iii:}
Since $b$ is bounded and uniformly Lipschitz in $x$, the hypotheses of
classical Lipschitz-flow theory
(see for example \cite[Remark~6.3]{ambrosio2004transport}) are satisfied.
That result provides a Borel function $J_{t,s}(x)>0$ such that:
\begin{itemize}
\item the Liouville-type identity \eqref{eq:liouv:lip} holds for
      a.e.\ $t$ and a.e.\ $x$;
\item the exponential bounds \eqref{eq:bd:jt} hold for every $s,t\in I$
      and a.e.\ $x$;
\item for each fixed $s,t$, $J_{t,s}(\cdot)$ coincides a.e.\ with
      $\det\nabla\Phi_{t,s}(\cdot)$.
\end{itemize}
For every fixed $s$, the map $t\mapsto J_{t,s}(x)$ is absolutely
continuous outside a null set of $x$. Integrating
\eqref{eq:liouv:lip} gives the logarithmic identity in the statement.

\medskip
\noindent\underline{Part iv:}
We now justify \eqref{eq:area:form} using the area formula. Recall
the following version of Federer~\cite[Thm.~3.2.5]{federer}: if
$F:\mathbb R^d\to\mathbb R^d$ is Lipschitz,
$A\subset\mathbb R^d$ is Lebesgue measurable, $J_F$ denotes the a.e.\
defined Jacobian determinant of $F$, and
$g:\mathbb R^d\to[0,\infty)$ is measurable, then
\begin{equation}\label{eq:feder:area}
\int_A g(F(x))J_F(x) dx
=
\int_{\mathbb R^d}g(y)N(F,A,y) dy,
\end{equation}
where
$N(F,A,y):=\#\{x\in A:F(x)=y\}$ is the multiplicity of $F$ on $A$ at $y$. We apply \eqref{eq:feder:area} with $F=\Phi_{t,s}$,
$A=\mathbb R^d$, and $J_F=J_{t,s}$. By Parts (i)--(ii), the map
$\Phi_{t,s}$ is a bi-Lipschitz homeomorphism with inverse $\Phi_{s,t}$,
hence it is bijective. Therefore $N(\Phi_{t,s},\mathbb R^d,y)=1$ for all
$y\in\mathbb R^d$, and \eqref{eq:feder:area} reduces to
\begin{equation}\label{eq:area:basic}
\int_{\mathbb R^d}g(\Phi_{t,s}(x))J_{t,s}(x) dx
=
\int_{\mathbb R^d}g(y) dy
\end{equation}
for every nonnegative measurable $g$.

Fix $s,t\in I$ and let
$\varphi:\mathbb R^d\to[0,\infty)$ be measurable. Since
$J_{t,s}(x)>0$ a.e., define
$g(y):=\varphi(y)/J_{t,s}(\Phi_{s,t}(y))$. Then $g$ is nonnegative, and
we can apply \eqref{eq:area:basic} with $g$ to get
\begin{align*}
\int_{\mathbb R^d}\varphi(\Phi_{t,s}(x)) dx
&=
\int_{\mathbb R^d}\varphi(\Phi_{t,s}(x))
\frac{J_{t,s}(x)}
{J_{t,s}(\Phi_{s,t}(\Phi_{t,s}(x)))} dx\stackrel{\eqref{eq:area:basic}}{=}
\int_{\mathbb R^d}
\frac{\varphi(y)}
{J_{t,s}(\Phi_{s,t}(y))} dy.
\end{align*}
This proves \eqref{eq:area:form} for nonnegative measurable $\varphi$.
The extension to integrable real-valued
$\varphi$ follows by applying the identity to $\varphi^+$ and
$\varphi^-$. Finally, differentiating
$\Phi_{t,s}\circ\Phi_{s,t}=\mathrm{Id}$ at points of differentiability
and using the fact that bi-Lipschitz maps preserve Lebesgue-null sets,
the chain rule and determinants give
$J_{t,s}(\Phi_{s,t}(y))J_{s,t}(y)=1$ a.e., which proves
\eqref{eq:jac:inv}.

\medskip
\noindent\underline{Part v:}
For $f\in C_c^\infty(\mathbb R^d)$, the pushforward formula gives us
\[
\int_{\mathbb R^d}f(y)q_{t,s}(y) dy
=
\int_{\mathbb R^d}f(\Phi_{t,s}(x))q(x) dx.
\]
Regarding the right-hand side, an application of \eqref{eq:area:form} with
$\varphi(y):=f(y)q(\Phi_{s,t}(y))$ leads to
\[
\int_{\mathbb R^d}f(\Phi_{t,s}(x))q(x) dx
=
\int_{\mathbb R^d}
f(y)\frac{q(\Phi_{s,t}(y))}
{J_{t,s}(\Phi_{s,t}(y))} dy.
\]
The last two equations imply \eqref{eq:push:expl:gen}. Now substituting $y=\Phi_{t,s}(x)$ and using $J_{t,s}(x)>0$ a.e.\ from \eqref{eq:bd:jt} give \eqref{eq:push:expl}.

\medskip
\noindent\underline{Part vi:}
By \eqref{eq:push:expl:gen} and the change of variables
$y=\Phi_{t,s}(x)$,
\[
\|q_{t,s}\|_{L^1}
\stackrel{\eqref{eq:push:expl:gen}}{=}
\int_{\mathbb R^d}
\frac{|q(\Phi_{s,t}(y))|}
{J_{t,s}(\Phi_{s,t}(y))}dy
=\int_{\mathbb R^d}|q(x)|dx
=\|q\|_{L^1}.
\]
\medskip
\noindent\underline{Part vii:}
Fix $s\in I$ and choose functions
$g_n\in C_c^\infty(\mathbb R^d)$ with $g_n\to q_s$ in $L^1$. Let
$g_{n,t,s}$ denote the density of $(\Phi_{t,s})_\#(g_n(x)dx)$.
We first show that
\[
\|g_{n,t,s}-g_n\|_{L^1}\longrightarrow0
\qquad\text{as }t\to s
\]
for every fixed $n$. Set $\tau:=|t-s|$. By
\eqref{eq:push:expl:gen},
$g_{n,t,s}(y)=g_n(\Phi_{s,t}(y))J_{s,t}(y)$ a.e. Hence,
\begin{align*}
\|g_{n,t,s}-g_n\|_{L^1}
&\le
\int_{\mathbb R^d}
\bigl|g_n(\Phi_{s,t}(y))-g_n(y)\bigr|J_{s,t}(y)dy
+\int_{\mathbb R^d}|g_n(y)|\bigl|J_{s,t}(y)-1\bigr|dy.
\end{align*}
Let $K_n:=\supp g_n$ and $K_n^*:=\{y\in\mathbb R^d:\operatorname{dist}(y,K_n)\le M_I|I|\}$. 
The displacement estimate in Part ii shows that the integrand in the
first integral vanishes outside $K_n^*$. Set
\[
\omega_{g_n}(r):=
\sup_{|x-y|\le r}|g_n(x)-g_n(y)|,
\]
then \eqref{eq:bd:jt} gives
\[
\int_{\mathbb R^d}
\bigl|g_n(\Phi_{s,t}(y))-g_n(y)\bigr|J_{s,t}(y)dy
\le
e^{D_I\tau}|K_n^*|\omega_{g_n}(M_I\tau).
\]
\[
\int_{\mathbb R^d}|g_n(y)|\bigl|J_{s,t}(y)-1\bigr|dy
\le
\bigl(e^{D_I\tau}-1\bigr)\|g_n\|_{L^1}.
\]
Both terms on the right-hand side converge to zero as $t\to s$. Since
$q_t=(\Phi_{t,s})_\#q_s$, \eqref{eq:pres:l1}  says
$\|q_t-g_{n,t,s}\|_{L^1}=\|q_s-g_n\|_{L^1}$, and consequently
\begin{align*}
\|q_t-q_s\|_{L^1}
&\le
\|q_t-g_{n,t,s}\|_{L^1}
+\|g_{n,t,s}-g_n\|_{L^1}
+\|g_n-q_s\|_{L^1}\\
&\stackrel{\eqref{eq:pres:l1}}{=}
2\|q_s-g_n\|_{L^1}
+\|g_{n,t,s}-g_n\|_{L^1}.
\end{align*}
First let $t\to s$ and then let $n\to\infty$. This proves continuity of
$t\mapsto q_t$ from $I$ to $L^1(\mathbb R^d)$.

\medskip
\noindent\underline{Part viii:}
Write $\Phi_t:=\Phi_{t,0}$ and $q_t:=q_{t,0}$. If
$\eta\in C_c^\infty(I\times\mathbb R^d)$, then for every
$x$ the map $t\mapsto\eta(t,\Phi_t(x))$ is absolutely continuous and
\[
\frac{d}{dt}\eta(t,\Phi_t(x))
=
\partial_t\eta(t,\Phi_t(x))
+
\nabla\eta(t,\Phi_t(x))\cdot b_t(\Phi_t(x))
\quad\text{for a.e.\ }t.
\]
Its derivative is bounded by
$\|\partial_t\eta\|_{L^\infty}+M_I\|\nabla\eta\|_{L^\infty}$, which is
integrable with respect to $|q(x)|dxdt$. Fubini's theorem, the compact support of $\eta$ in the time
variable, and the pushforward formula now give the claimed continuity equation. This completes the proof. 
\end{proof}

\subsection{KL first variation formula}
\label{sec:bl:ext}

\begin{assumption}\label{ass:phi:bl}
$\phi\in \mathrm{BL}(\R^d;\R^d)$, i.e.  $\phi$ is bounded and
        Lipschitz on $\R^d$.
\end{assumption}

\begin{assumption}\label{ass:bl:ext}
Fix an integer $m\geq 1$. We assume:
\begin{enumerate}[label=(A\arabic*),leftmargin=2.4em]
 \item $U\in \mathrm{Lip}_{\mathrm{loc}}(\mathbb R^d)$ and $U_0\in C^1(\mathbb R^d)$. Set $a:=e^{U-U_0},
 p:=Z^{-1}e^{-U}$ and
 $p_0:=Z_0^{-1}e^{-U_0}$. There exists a constant $C_U>0$ such that
\[
|\nabla U_0(x)|
\le C_U\bigl(1+|x|^m\bigr)
\qquad\text{for all }x\in\mathbb R^d,
\]
and
\[
|\nabla U(x)|
\le C_U\bigl(1+|x|^m\bigr),
\qquad
|U(x)|
\le C_U\bigl(1+|x|^{m+1}\bigr)
\qquad\text{for a.e. }x\in\mathbb R^d.
\]

  \item Uniform Lipschitz anchor:
        $a\in W^{1,\infty}(\R^d)$.
        Let $A_\infty:=\|a\|_{L^\infty}<\infty$ and $L_a:=\|\nabla a\|_{L^\infty}<\infty$.

  \item The initial density $q\in L^1(\mathbb R^d)$. 

 \item Growth and integrability for the target:
        \begin{itemize}
        \item[(a)] $q$ has a finite $m$-th moment: $\int_{\R^d}\bigl(1+|x|^m\bigr) q(x) dx < \infty$. 
        \item[(b)] The initial KL divergence is finite: $ \KL(q\|p) = \int_{\R^d} q(x)\log\frac{q(x)}{p(x)} dx < \infty$. 
        \end{itemize}
\end{enumerate}
\end{assumption}

\subsubsection{Proof of Theorem~\ref{theo:bl:ext}}
Under Assumptions \ref{ass:phi:bl} and \ref{ass:bl:ext}, the anchored vector
field $\psi:=a\phi$
belongs to $W^{1,\infty}(\mathbb R^d;\mathbb R^d)$. In particular,
$\psi$ admits a bounded globally Lipschitz representative, and
$\div\psi\in L^\infty(\mathbb R^d)$. Therefore, the ODE
\begin{equation}
\label{bl:def:fl}
\dot\Phi_t(x)=\psi(\Phi_t(x)),\qquad \Phi_0(x)=x,
\end{equation}
is well-posed, and the existence and uniqueness of
Lemma~\ref{lem:push:timed}, applied with the time-independent field
$b_s\equiv\psi$ on each compact time interval,
gives the associated global flow $(\Phi_t)_{t\in\mathbb R}$. We define
\begin{equation}
\label{bl:def:push:den}
q_t:=(\Phi_t)_\#q.
\end{equation}
The continuity equation in the statement of Theorem~\ref{theo:bl:ext} follows from Lemma~\ref{lem:push:timed}(viii), applied
with $b_s\equiv\psi$. Fix $t_0>0$, take $I=[-t_0,t_0]$ and write
\[
F(t):=\KL(q_t\|p).
\]
We next show absolute continuity of $F$ on $[-t_0,t_0]$.

Let $J_t:=J_{t,0}$ be the Jacobian from
Lemma~\ref{lem:push:timed}(iii) for the time-independent field
$b_s\equiv\psi$.
Set
\[
\mathcal J(t):=-\int_{\R^d}q(x)\log J_t(x) dx,
\qquad
\mathcal U(t):=
\int_{\R^d}q(x)\bigl(U(\Phi_t(x))-U(x)\bigr) dx.
\]

\medskip
\noindent\underline{Step 1:} Via the density formula \eqref{eq:push:expl} in
Lemma~\ref{lem:push:timed}(v) which gives $q_t(\Phi_t(x))=\frac{q(x)}{J_t(x)}$ a.e. and $p=Z^{-1}e^{-U}$, we obtain for a.e. $x$ and on the set $\{q>0\}$ that  $\log\frac{q_t(\Phi_t(x))}{p(\Phi_t(x))}
=
\log q(x)-\log J_t(x)+U(\Phi_t(x))+\log Z
=
\log\frac{q(x)}{p(x)}
+U(\Phi_t(x))-U(x)-\log J_t(x).$
Combining this with the pushforward formula gives 
\begin{align*}
F(t)
&=
\int_{\R^d}q_t(y)\log\frac{q_t(y)}{p(y)}dy=
\int_{\R^d}q(x)
\log\frac{q_t(\Phi_t(x))}{p(\Phi_t(x))}dx\\
&=
\int_{\R^d}q(x)\log\frac{q(x)}{p(x)}dx
+
\int_{\R^d}q(x)
\Bigl[
U(\Phi_t(x))-U(x)-\log J_t(x)
\Bigr]dx.
\end{align*}
and hence
\begin{align}
\label{bl:eq:kl:trans}
F(t)-F(0)
=
\int_{\R^d}q(x)
\Bigl[
U(\Phi_t(x))-U(x)-\log J_t(x)
\Bigr]dx.
\end{align}

\medskip
\noindent\underline{Step 2:}
By Lemma~\ref{lem:push:timed}(iii), for a.e.\ $x$, the map
$t\mapsto J_t(x)$ is positive, absolutely continuous and $\log J_t(x)
=
\int_0^t(\div\psi)(\Phi_s(x)) ds $. 
Moreover, since $\div\psi\in L^\infty(\R^d)$, the integrand of $\mathcal J(t)$ is bounded by $|I|
\|\div\psi\|_{L^\infty}q(x)$
which is integrable. Then Fubini's theorem and the pushforward formula give us
\begin{align}
\label{bl:eq:j:dbl:int}
\mathcal J(t)
&=
-\int_0^t\int_{\R^d}
q_s(y)\div\psi(y) dy ds.
\end{align}
Thus, $\mathcal J$ is absolutely continuous on $I$.

 \medskip
\noindent\underline{Step 3:}
Next, let $M_*:=\|\psi\|_{L^\infty}$. Lemma~\ref{lem:push:timed}(ii) gives
\begin{equation}
\label{bl:eq:direc:disp}
|\Phi_t(x)-x|\le M_*|t|,
\qquad t\in I.
\end{equation}
In addition, \eqref{bl:eq:direc:disp} gives
\begin{equation}
\label{bl:eq:mom:prop}
\sup_{t\in I}\int_{\R^d}(1+|y|^m)q_t(y) dy
=
\sup_{t\in I}\int_{\R^d}(1+|\Phi_t(x)|^m)q(x) dx
<\infty.
\end{equation}
The chain rule gives for $q$-a.e.\ $x$ that $\frac{d}{dt}U(\Phi_t(x))
=
\nabla U(\Phi_t(x))\cdot \psi(\Phi_t(x))$,  which implies for $q$-a.e.\ $x$ and every $s,t\in I$,
\[
U(\Phi_t(x))-U(\Phi_s(x))
=
\int_s^t
\nabla U(\Phi_r(x))\cdot \psi(\Phi_r(x)) dr.
\]
Regarding the integrand on the right-hand side, the boundedness of $\psi$ and the growth bound on $\nabla U$ in Assumption~\ref{ass:bl:ext} give us $\bigl|
\nabla U(\Phi_r(x))\cdot \psi(\Phi_r(x))
\bigr|
\le
C\bigl(1+|\Phi_r(x)|^m\bigr)$ a.e. $r\in I$, which is integrable with respect to $q(x)dx dr$ per \eqref{bl:eq:mom:prop}. Thus, an application of Fubini's theorem leads to
\begin{align}
\label{bl:eq:u:dbl:int}
\mathcal U(t)-\mathcal U(s)
&=
\int_s^t\int_{\R^d}
q(x)\nabla U(\Phi_r(x))\cdot \psi(\Phi_r(x)) dx dr=
\int_s^t\int_{\R^d}
q_r(y)\nabla U(y)\cdot \psi(y) dy dr.
\end{align}
In particular, $\mathcal U$ is absolutely continuous on $I$.

\medskip
\noindent\underline{Step 4:}
By \eqref{bl:eq:kl:trans},
\[
F(t)=F(0)+\mathcal U(t)+\mathcal J(t).
\]
The previous steps show that the last two terms are finite and absolutely
continuous on $I$. They also complete the integrability
justification for \eqref{bl:eq:kl:trans}. Combining \eqref{bl:eq:kl:trans}, \eqref{bl:eq:j:dbl:int} and
\eqref{bl:eq:u:dbl:int}, we can conclude that $F(t)$ is absolutely continuous on $I$ and that 
\begin{align}
\label{eq:kl:der:int:form}
   & F(t)-F(s)
   \stackrel{\eqref{bl:eq:kl:trans},\eqref{bl:eq:j:dbl:int},\eqref{bl:eq:u:dbl:int}}{=}
\int_s^t\int_{\R^d}q_r(x)
\bigl(\nabla U(x)\cdot \psi(x)-\div\psi(x)\bigr)dx dr=
\int_s^t\int_{\R^d}q_r(x)g(x)dx dr,\nn\\
&\hspace{15em}\text{for}\quad g(x):=-a(x)\bigl(\div\phi(x)-\nabla U_0(x)\cdot\phi(x)\bigr)
\end{align}
The last line is due to $\psi=a\phi$ plus the weak chain and product
rules which give $\nabla a=a(\nabla U-\nabla U_0)$ and $\div(a\phi)=\nabla a\cdot\phi+a\div\phi$ a.e.

\medskip
\noindent\underline{Step 5:} Assumptions~\ref{ass:phi:bl}, \ref{ass:bl:ext}(A1), and
\ref{ass:bl:ext}(A2) give
$|g(x)|\le C(1+|x|^m)$ a.e. We will show that
$r\mapsto\int_{\R^d}q_r(x)g(x)dx$ is continuous at $r=0$. By \eqref{bl:eq:direc:disp}, $|\Phi_r(x)-x|\le D:=M_*t_0$ for
$|r|\le t_0$. Hence $|\Phi_r(x)|>R$ implies $|x|>R-D$, and also
$1+|\Phi_r(x)|^m\le 1+(|x|+D)^m\le C(1+|x|^m)$. Using
$q_r=(\Phi_r)_\#q$, we obtain
\begin{align*}
\int_{|y|>R}(1+|y|^m)q_r(y)dy
&=
\int_{\{|\Phi_r(x)|>R\}}
(1+|\Phi_r(x)|^m)q(x)dx\le
C\int_{|x|>R-D}(1+|x|^m)q(x)dx.
\end{align*}
Then for $R>D$,
\begin{align*}
&\left|\int_{\R^d}(q_r(x)-q(x))g(x)dx\right|
\le
\int_{|x|\le R}|q_r(x)-q(x)||g(x)|dx+
\int_{|x|>R}(q_r(x)+q(x))|g(x)|dx\\
&\hspace{6em}\le
C(1+R^m)\int_{|x|\le R}|q_r(x)-q(x)|dx+
C\int_{|x|>R}(1+|x|^m)(q_r(x)+q(x))dx\\
&\hspace{6em}\le
C(1+R^m)\|q_r-q\|_{L^1}+
C\int_{|x|>R-D}(1+|x|^m)q(x)dx.
\end{align*}
Lemma~\ref{lem:push:timed}(vii) gives
$\|q_r-q\|_{L^1}\to0$. In addition,
Assumption~\ref{ass:bl:ext}(A4)(a) says $\int_{|x|>R-D}(1+|x|^m)q(x)dx\to 0$ as $R\to\infty$. Thus,  $\int_{\R^d}q_r(x)g(x)dx
\longrightarrow
\int_{\R^d}q(x)g(x)dx$ which is continuity of $r\mapsto\int_{\R^d}q_r(x)g(x)dx$ at $r=0$. Then per \eqref{eq:kl:der:int:form}, $F(t)\in C^1(I)$ and it follows that $F'(0)
=
-\mathbb E_q\!\left[
a\bigl(\div\phi-\nabla U_0\cdot\phi\bigr)
\right]$. This completes the proof.


\subsection{Optimal direction in the RKHS} 
\label{sec:opt:dir}

\begin{assumption}[Kernels]
\label{ass:rkhs}
Let $k:\mathbb R^d\times\mathbb R^d\to\mathbb R$ be a reproducing kernel with RKHS $\mathcal H$.
We write $\partial_{x_r}k$ and $\partial_{y_r}k$ for the $r$-th coordinate derivatives in the first and second arguments, respectively. These derivatives are taken before evaluation; $\nabla_1k=(\partial_{x_r}k)_{r=1}^d$ and $\nabla_2k=(\partial_{y_r}k)_{r=1}^d$.
Assume:

\begin{enumerate}[label=(K\arabic*),leftmargin=2.4em]
\item For each $i\in\{1,\dots,d\}$,  $\partial_{x_i}k(x,\cdot)$ belongs to $\mathcal H$ for every $x\in\mathbb R^d$.

\item  there is a constant $B>0$ such that for all $x,y\in\mathbb R^d$ and all coordinates
$r,s\in\{1,\dots,d\}$,
\[
|k(x,y)|+|\partial_{x_r}k(x,y)|+|\partial_{y_s}k(x,y)|
+|\partial_{x_s}\partial_{x_r}k(x,y)|
+|\partial_{y_s}\partial_{x_r}k(x,y)|
\le B;
\]

 \end{enumerate}
    
\end{assumption}

\begin{lemma}
\label{lem:k2:k3:der:repr}
Assume Assumption~\ref{ass:rkhs} holds. Then:
\begin{enumerate}[label=(\roman*),leftmargin=2em]
\item For each $i\in[d]$ and
each $f\in\mathcal H$, the partial derivative $\partial_i f(x)$ exists for all
$x\in\mathbb R^d$ and satisfies
\begin{equation}
\label{eq:deriv:repr}
\partial_i f(x)=\left\langle f,\partial_{x_i}k(x,\cdot)\right\rangle_{\mathcal H}.
\end{equation}

\item $\mathcal H^d$ can be embedded continuously into
$\mathrm{BL}(\mathbb R^d;\mathbb R^d)$. In particular, for every
$\phi\in\mathcal H^d$,
\[
  \|\phi\|_{L^\infty}\le \sqrt B\|\phi\|_{\mathcal H^d},
  \qquad
  \mathrm{Lip}(\phi)\le \sqrt{dB}\|\phi\|_{\mathcal H^d},
\]
and hence $\|\phi\|_{\mathrm{BL}}\le (\sqrt B+\sqrt{dB})\|\phi\|_{\mathcal H^d}$.
\end{enumerate}
\end{lemma}

\begin{proof} 
\underline{Part i:}
Fix $1\leq i\leq d$ and $x\in\mathbb R^d$. For $h\neq 0$, set $q_h:=\frac{k(x+h e_i,\cdot)-k(x,\cdot)}{h}$. An application of the mean-value
theorem and (K2) give
\begin{align}
\label{eq:bound:qh:norm}
\|q_h\|_{\mathcal H}^2
=
\frac{k(x+h e_i,x+h e_i)-k(x+h e_i,x)-k(x,x+h e_i)+k(x,x)}{h^2}
\le B.
\end{align}
Thus $(q_h)_{h\to0}$ is bounded in $\mathcal H$. For each
$y\in\mathbb R^d$, the reproducing property and (K1) give
\[
\left\langle q_h,k(y,\cdot)\right\rangle_{\mathcal H}
=\frac{k(x+h e_i,y)-k(x,y)}{h}
\longrightarrow
\partial_{x_i}k(x,y)
=\left\langle\partial_{x_i}k(x,\cdot),k(y,\cdot)\right\rangle_{\mathcal H}.
\]
Finite linear combinations of $k(y,\cdot)$ are dense in
$\mathcal H$. Together with the uniform bound on $\|q_h\|_{\mathcal H}$ at \eqref{eq:bound:qh:norm}, this extends the above convergence of inner products to every \(f\in\mathcal H\). Then for $f\in\mathcal H$,
\[
\partial_if(x)=\frac{f(x+h e_i)-f(x)}{h}
=\langle f,q_h\rangle_{\mathcal H}
\longrightarrow
\left\langle f,\partial_{x_i}k(x,\cdot)\right\rangle_{\mathcal H},
\]
which is \eqref{eq:deriv:repr}.

\medskip
\noindent\underline{Part ii:}
We show in this step $\mathcal H^d$ can be embedded continuously into $\mathrm{BL}(\mathbb R^d;\mathbb R^d)$. Fix $f\in\mathcal H$. By (K2) in Assumption~\ref{ass:rkhs}, $\sup_x k(x,x)\le B$ so that  $  |f(x)|
  =|\langle f,k(x,\cdot)\rangle_{\mathcal H}|
  \le \|f\|_{\mathcal H} \|k(x,\cdot)\|_{\mathcal H}
  = \|f\|_{\mathcal H} \sqrt{k(x,x)}
  \le \sqrt B\|f\|_{\mathcal H}$. Consequently, $\|f\|_{L^\infty}\le \sqrt B\|f\|_{\mathcal H}$.

To bound $\mathrm{Lip}(f)$, fix $x,y\in\mathbb R^d$ and set
\[
x^{(0)}:=y,
\qquad
x^{(i)}:=(x_1,\dots,x_i,y_{i+1},\dots,y_d),
\quad i=1,\dots,d.
\]
Now for each $i\in\{1,\dots,d\}$ and each $x\in\R^d$, $\|\partial_{x_i}k(x,\cdot)\|_{\mathcal H}^2
=
\partial_{y_i}\partial_{x_i}k(x,x)$, and therefore by (K2),
\begin{align}
\label{eq:bd:first:arg}
\sup_{x\in\R^d}\sum_{i=1}^d \|\partial_{x_i}k(x,\cdot)\|_{\mathcal H}^2 \le dB.
\end{align}
By \eqref{eq:deriv:repr}, Cauchy-Schwarz's inequality, and (K2), \begin{align}
\label{eq:bound:part:fz}
    |\partial_i f(z)|
\le \|f\|_{\mathcal H}
\|\partial_{x_i}k(z,\cdot)\|_{\mathcal H}
\le \sqrt B\|f\|_{\mathcal H}
\end{align}. 
The mean-value theorem then leads to
\begin{align*}
|f(x)-f(y)|
&\le
\sum_{i=1}^d
|f(x^{(i)})-f(x^{(i-1)})|\stackrel{\eqref{eq:bound:part:fz}}\le
\sqrt B\|f\|_{\mathcal H}\sum_{i=1}^d|x_i-y_i|
\le
\sqrt{dB}\|f\|_{\mathcal H}|x-y|.
\end{align*}
Thus, $\mathrm{Lip}(f)\le \sqrt{dB}\|f\|_{\mathcal H}$. Applying this and $\|f\|_{L^\infty}\le \sqrt B\|f\|_{\mathcal H}$ from earlier componentwise to $\phi\in\mathcal H^d$ yields the desired bounds on $\|\phi\|_{L^\infty}$ and $\mathrm{Lip}(\phi)\le \sqrt{dB}\|\phi\|_{\mathcal H^d}$. 
\end{proof}

\subsubsection{Proof of Theorem~\ref{theo:opt:dir}}
\underline{Step 1:} Lemma~\ref{lem:k2:k3:der:repr}(ii) shows that $\mathcal H^d$ can be embedded continuously into $\mathrm{BL}(\mathbb R^d;\mathbb R^d)$. This combined with the existence and uniqueness  of Lemma~\ref{lem:push:timed} (which requires a bounded Lipschitz vector field) imply \eqref{def:fl:repr:sp} is well defined.

\medskip
\underline{Step 2:}  Define the linear functional $\mathcal L(\phi) := \mathbb E_q\!\big[a(\mathrm{div} \phi-\nabla U_0\cdot \phi)\big]$. Then due to Step 1 and Theorem~\ref{theo:bl:ext}, we have that the map $t\mapsto \mathrm{KL}(q_t\|p)$ is differentiable at $t=0$ and 
\begin{align}\label{eq:der:kl:h} \left.\frac{d}{dt}\mathrm{KL}(q_t\|p)\right|_{t=0} = -\mathcal L(\phi).
\end{align}

\medskip
\underline{Step 3:} Write $\mathrm{div} \phi=\sum_{i=1}^d \partial_i\phi_i$.
By \eqref{eq:deriv:repr}, $ \partial_i\phi_i(x)=\langle \phi_i,\partial_{x_i}k(x,\cdot)\rangle_{\mathcal H}$ and $\phi_i(x)=\langle \phi_i,k(x,\cdot)\rangle_{\mathcal H}$. Define $\mathcal H$-valued maps
\[
  G_i(x):=a(x) \partial_{x_i}k(x,\cdot),
  \qquad
  H_i(x):=a(x) (\partial_i U_0)(x) k(x,\cdot).
\]
By \eqref{eq:bd:first:arg} and $a\in L^\infty$,
\[
  \|G_i(x)\|_{\mathcal H}
  \le \|a\|_{L^\infty} \|\partial_{x_i}k(x,\cdot)\|_{\mathcal H}
  \stackrel{\eqref{eq:bd:first:arg}}{\le} \|a\|_{L^\infty} \sqrt{dB},
\]
so $\mathbb E_q[\|G_i(X)\|_{\mathcal H}]<\infty$.
Also, by (K2) in Assumption~\ref{ass:rkhs} and Assumption~\ref{ass:bl:ext},
\[
  \|H_i(x)\|_{\mathcal H}
  \le \|a\|_{L^\infty} |(\partial_i U_0)(x)| \|k(x,\cdot)\|_{\mathcal H}
  \le \|a\|_{L^\infty} C_U(1+|x|^m) \sqrt B,
\]
whose expected value with respect to $q$ is finite since $q$ has finite $m$-th moment. Hence $G_i(X)$ and $H_i(X)$ are Bochner integrable in $\mathcal H$.

Using the property that bounded linear functionals commute with Bochner integration (see for example \cite[Theorem A2.3]{vanNeerven1992}), we have for each fixed $\phi_i\in\mathcal H$:
\[
  \mathbb E_q\!\big[a(X)\partial_i\phi_i(X)\big]
  =\mathbb E_q\!\big[\langle \phi_i,G_i(X)\rangle_{\mathcal H}\big]
  =\left\langle \phi_i,\mathbb E_q[G_i(X)]\right\rangle_{\mathcal H},
\]
and similarly
\[
  \mathbb E_q\!\big[a(X)(\partial_i U_0)(X)\phi_i(X)\big]
  =\mathbb E_q\!\big[\langle \phi_i,H_i(X)\rangle_{\mathcal H}\big]
  =\left\langle \phi_i,\mathbb E_q[H_i(X)]\right\rangle_{\mathcal H}.
\]
Summing over $i$ yields
\begin{align}
\label{eq:lphi}
  \mathcal L(\phi)
  =\sum_{i=1}^d \langle \phi_i,\mathbb E_q[G_i(X)]-\mathbb E_q[H_i(X)]\rangle_{\mathcal H}
  =\langle \phi,h\rangle_{\mathcal H^d},
\end{align}
with $h_i$ in \eqref{def:opt:dir:comp}.

\underline{Step 4:} By \eqref{eq:lphi} and the Cauchy-Schwarz inequality,
\[
  \mathcal L(\phi)\stackrel{\eqref{eq:lphi}}{=}\langle \phi,h\rangle_{\mathcal H^d}
  \le \|\phi\|_{\mathcal H^d} \|h\|_{\mathcal H^d}.
\]
Taking the supremum over $\|\phi\|_{\mathcal H^d}\le 1$ gives
$\sup_{\phi\in\mathbb B_{\mathcal H^d}}\mathcal L(\phi)\le \|h\|_{\mathcal H^d}$, and equality holds by choosing
the $\phi^\star$ in the theorem statement.
Using \eqref{eq:der:kl:h} yields the steepest KL decrease and also the De Bruijn-type identity.

\section{Anchored Stein variational gradient descent algorithm}
\label{sec:anch:svgd}

Recall that $p_0(x)\propto e^{-U_0(x)}$ and $a(x)=e^{U(x)-U_0(x)}$. Define the anchored interaction kernel
\begin{align}
\label{eq:zeta:a}
\zeta_a(x,y):=a(x)a(y)\Bigl(\nabla_1 k(y,x)-k(y,x)\nabla U_0(y)\Bigr),
\end{align}
and the particle drift
\begin{equation}
\label{def:disc:drift}
F(x):=\bigl(F_1(x),\dots,F_N(x)\bigr),
\qquad
F_i(x):=\frac1N\sum_{j=1}^N \zeta_a(x_i,x_j),
\qquad x=(x_1,\dots,x_N)\in(\R^d)^N.
\end{equation}
Then the particle update in Algorithm~\ref{alg:anch:svgd} can be written as
\begin{equation}
\label{eq:disc:step}
x_i^{\ell+1}=x_i^\ell+\varepsilon_\ell F_i(x^\ell),
\qquad
i=1,\dots,N,\quad \ell=0,1,\dots,L-1,
\end{equation}
where
\[
x^\ell=\bigl(x_1^\ell,\dots,x_N^\ell\bigr)\in(\R^d)^N.
\]

\begin{assumption}\label{ass:stab}
Suppose that 
\begin{enumerate}[label=(S\arabic*),leftmargin=2.4em]
\item there exist constants $A_0>0$, $\alpha_0\in[0,\tfrac12)$, and $C_{U_0}>0$ such that
\[
\inf_{x\in\mathbb R^d}U_0(x)>0,
\qquad
|\nabla U_0(x)|\le A_0 U_0(x)^{\alpha_0},
\qquad
\|\nabla^2 U_0(x)\|_{\mathrm{op}}\le C_{U_0},
\qquad x\in\mathbb R^d.
\]
\item the initial density  $q $ satisfies $\int_{\R^d} U_0(x) q(x) dx<\infty$. 
\end{enumerate}
\end{assumption}

\begin{remark}
    The main difference between our derivation of the convergence rate for anchored SVGD (Algorithm~\ref{alg:anch:svgd}) and the improved convergence rate analysis of \cite{banerjeeetal2025}  lies in the choice of Lyapunov functional. Therein, the Lyapunov functional is taken to be $\KL\!\big(p^N(t)\|p^{\otimes N}\big)$, where $p\propto e^{-U}$ is the target distribution. In the anchored setting, the natural choice is instead $\KL\!\big(p^N(t)\|p_0^{\otimes N}\big)$, where $p_0\propto e^{-U_0}$ is the reference density. In particular, the proof of Theorem~\ref{theo:disc:bd:ksd2} below will rely on Proposition~\ref{prop:disc:bd:ksd:squar:all:norm}, and with $\KL\!\big(p^N(t)\|p_0^{\otimes N}\big)$ several otherwise complicated remainder terms vanish, leading to a substantially cleaner estimate.
\end{remark}


The next proposition is a descent estimate for anchored SVGD (Algorithm \ref{alg:anch:svgd}). It is written for a general positive density $\tilde p$, and later we will specialize this estimate to the choice $\tilde p=p_0\propto e^{-U_0}$, which is the natural Lyapunov reference in the anchored setting. This is the key estimate used in the later proof of  Theorem~\ref{theo:disc:bd:ksd2}.
\begin{proposition} 
\label{prop:disc:bd:ksd:squar:all:norm}
Fix $N,L\in\mathbb N$, $K>0$, and $\eta>0$, and suppose
$\varepsilon_r\equiv\eta$ for $r=0,\dots,L-1$. Assume 
Assumption~\ref{ass:bl:ext}(A2), Assumption~\ref{ass:rkhs}(K2), and Assumption~\ref{ass:stab}(S1). Let $\tilde p$ be a positive reference density on $\R^d$, and define
\[
\tilde u(x):=\log \tilde p(x).
\]
Let $g(0)$ denote the density of the initial joint law, and define the truncated initialization $g_K(0)$ by
\begin{align}
\label{def:disc:init}
g_K(0,x)
&:=
\frac{g(0,x)\mathbf 1_{G_K}(x)}{\mu_0(G_K)},
\qquad
\mu_0(G_K):=\int_{G_K} g(0,x) dx>0, \notag\\
G_K
&:=
\left\{x\in(\mathbb R^d)^N:\ \frac1N\sum_{i=1}^N U_0(x_i)\le K\right\}.
\end{align}
For $\ell=1,\dots,L$, define the law initialized from this truncation by
\begin{equation}
\label{def:disc:gk:ell}
T_r(x):=x+\varepsilon_rF(x),
\qquad r=0,\dots,L-1,
\qquad
g_K(\ell):=
\bigl(T_{\ell-1}\circ\cdots\circ T_0\bigr)_\#g_K(0). 
\end{equation}
We will use $g_K(\ell)$ also for its density.

Assume 
\begin{align}\label{assum:algo:parti}\widetilde u=\log\widetilde p\in C^2(\mathbb R^d),\qquad F\in W_{\mathrm{loc}}^{1,\infty}((\mathbb R^d)^N;(\mathbb R^d)^N),\qquad\mathrm{KL}(g_K(0)\|\widetilde p^{\otimes N})<\infty\end{align}
Set
\[
\phi_t(x):=x+tF(x),\qquad
A_t(x):=I+tJF(x),\qquad
\widetilde u_N(x):=\sum_{i=1}^N\widetilde u(x_i).
\]
For every $0\le\ell\le L-1$, assume for $t\in[0,\eta]$ that 
\begin{align}
\label{assum:pos:det}
    \det A_t(x)>0\qquad g_K(\ell)-\text{a.e. $x$}
\end{align}
and
\begin{equation}
\label{ass:disc:diff:env}  
\int g_K(\ell,x)
\Bigl(|\log g_K(\ell,x)|+\mathcal D_\ell(x)\Bigr) dx<\infty,
\end{equation}
where
\begin{align}
\label{def:dell}
\mathcal D_\ell(x):=
\sup_{0\le t\le\eta}\Bigl\{
&|\log\det A_t(x)|
+|\widetilde u_N(\phi_t(x))|\nn\\
&+\|A_t(x)^{-1}JF(x)\|_{\mathrm{HS}}
+|\nabla\widetilde u_N(\phi_t(x))| |F(x)|\nn\\
&+\|A_t(x)^{-1}JF(x)\|_{\mathrm{HS}}^2
+\|\nabla^2\widetilde u_N(\phi_t(x))\|_{\mathrm{op}}|F(x)|^2
\Bigr\},
\end{align}
and $\|\cdot\|_{\mathrm{HS}}$ is the Hilbert-Schmidt matrix norm.

Then we have
\begin{align}
\label{eq:disc:ksd2e}
&\mathbb E_{g_K(0)}
\Biggl[
\frac1L\sum_{\ell=0}^{L-1}
\KSD_a^2\bigl(\mu_\ell^N\|p\bigr)
\Biggr]
\notag\\
&\qquad\le
\frac{\mathrm{KL}\bigl(g_K(0)\|\tilde p^{\otimes N}\bigr)}{NL\eta}
+\frac{1}{NL}\sum_{\ell=0}^{L-1}
\Bigl(
R_3^{\mathrm{disc}}(\ell)
+
R_{\mathrm{diag}}^{\mathrm{disc}}(\ell)
\Bigr)
+\frac{1}{NL\eta}\sum_{\ell=0}^{L-1}\int_0^\eta (\eta-t)E''(t) dt.
\end{align}

Here $E''(t)$ is given explicitly in \eqref{eq:disc:e2der} and 
\begin{align*}
R_3^{\mathrm{disc}}(\ell)
&:=
\frac1N\sum_{k,j=1}^N
\int_{(\mathbb R^d)^N} g_K(\ell,x) \Bigl(R_1(x_k,x_j)+R_2(x_k,x_j)\Bigr) dx;\\
R_{\mathrm{diag}}^{\mathrm{disc}}(\ell)
&:=
\frac1N\sum_{k=1}^N
\int_{(\mathbb R^d)^N} g_K(\ell,x) D_a(x_k) dx.
\end{align*}
Moreover, 
\begin{align}
\label{def:w:rem:da}
w_{\tilde p}(x)&:=\nabla \tilde u(x)+\nabla U_0(x);\nn\\
R_1(x,y)
&:=
a(x)a(y)\Bigl(
-w_{\tilde p}(x)\cdot\nabla_1 k(y,x)
+k(y,x) w_{\tilde p}(x)\cdot\nabla U_0(y)
\Bigr);\nn\\
R_2(x,y)
&:=
a(y)\nabla a(x)\cdot\Big(k(y,x)\nabla U_0(y)-\nabla_1 k(y,x)\Big);\nn\\
D_a(x)
&:=
a(x)\nabla a(x)\cdot\big(k(x,x)\nabla U_0(x)-\nabla_1 k(x,x)\big)\nonumber\\
&\quad
+a(x)^2\Big(
\div_x\big(k(x,x)\nabla U_0(x)\big)
-\nabla U_0(x)\cdot\nabla_2 k(x,x)
-\Delta_1 k(x,x)
\Big).
\end{align}

\end{proposition}

\begin{proof}
Fix $0\le \ell\le L-1$ and define
\begin{equation}
\label{eq:anc:phi:nu}
\phi_t(x):=x+t F(x),\qquad \nu(t):=(\phi_t)_{\#}g_K(\ell)
\qquad
t\in[0,\eta].
\end{equation}
Denote by $q(t,\cdot)$ the density of $\nu(t)$. Then \eqref{eq:disc:step} implies that 
\begin{equation}
\label{eq:disc:end:law}
\nu(0)=g_K(\ell),
\qquad
\nu(\eta)=g_K(\ell+1).
\end{equation}
Define
\begin{equation}
\label{def:disc:kl:time}
E(t):=\mathrm{KL}\bigl(q(t)\|\tilde p^{\otimes N}\bigr)=
\int_{\mathbb R^{dN}} q(t,\xi)\log\frac{q(t,\xi)}{\tilde p^{\otimes N}(\xi)} d\xi,
\qquad
t\in[0,\eta],
\end{equation}
and notice by \eqref{eq:disc:end:law}, $E(0)=\mathrm{KL}\bigl(g_K(\ell)\|\tilde p^{\otimes N}\bigr)$ and $E(\eta)=\mathrm{KL}\bigl(g_K(\ell+1)\|\tilde p^{\otimes N}\bigr)$.


\medskip
\noindent\underline{Step 1:}
Using the change of variables
$\xi=\phi_t(x)$ and $q\bigl(t,\phi_t(x)\bigr)
=
\det(J\phi_t(x))^{-1}g_K(\ell,x)$ (similar to \eqref{eq:push:expl:gen}),
we obtain
\begin{equation}
\label{eq:disc:kl:pull}
E(t)
=
\int_{(\mathbb R^d)^N}
g_K(\ell,x)
\Bigl[
\log g_K(\ell,x)
-\log\det A_t(x)
-\widetilde u_N(\phi_t(x))
\Bigr]dx,
\end{equation}
where $A_t(x)=J\phi_t(x)=I+tJF(x)$.

Set $m:=dN$. Since
\begin{equation}
\label{eq:disc:part:phi}
\partial_t\phi_t(x)=F(x),
\qquad
\partial_tA_t(x)=JF(x),
\end{equation}
Jacobi's formula $\frac{d}{dt}\log\det A_t
=
\mathrm{tr}\!\bigl(A_t^{-1}\dot A_t\bigr)$ and the chain rule give
\begin{align*}
\frac{d}{dt}
\Bigl[-\log\det A_t(x)-\widetilde u_N(\phi_t(x))\Bigr]
&=
-\operatorname{tr}\!\bigl(A_t(x)^{-1}JF(x)\bigr)
-\bigl\langle\nabla\widetilde u_N(\phi_t(x)),F(x)\bigr\rangle;\\
\frac{d^2}{dt^2}
\Bigl[-\log\det A_t(x)-\widetilde u_N(\phi_t(x))\Bigr]
&=
\operatorname{tr}\!\Bigl(
A_t(x)^{-1}JF(x)A_t(x)^{-1}JF(x)
\Bigr)\\
&\quad
-\bigl\langle
F(x),\nabla^2\widetilde u_N(\phi_t(x))F(x)
\bigr\rangle .
\end{align*}
Consequently, for $t\in[0,\eta]$,
\begin{align*}
&\left|
\log g_K(\ell,x)
-\log\det A_t(x)
-\widetilde u_N(\phi_t(x))
\right|
\le
|\log g_K(\ell,x)|+\mathcal D_\ell(x),\\
&\left|
\frac{d}{dt}
\Bigl[-\log\det A_t(x)-\widetilde u_N(\phi_t(x))\Bigr]
\right|
\le
(\sqrt m+1)\mathcal D_\ell(x),\\
&\left|
\frac{d^2}{dt^2}
\Bigl[-\log\det A_t(x)-\widetilde u_N(\phi_t(x))\Bigr]
\right|
\le
\mathcal D_\ell(x).
\end{align*}
After multiplication by $g_K(\ell,x)$, these three estimates bound,
respectively, the integrand in \eqref{eq:disc:kl:pull} and its first
two derivatives with respect to $t$. These bounds are integrable by
condition~\eqref{ass:disc:diff:env}. Leibniz's integral rule
therefore allows us to differentiate twice under the integral sign in
\eqref{eq:disc:kl:pull}. 

\medskip
\noindent\underline{Step 2:} We compute $E'(t)$ in this step. Assumption~\ref{ass:bl:ext}(A2) ensures that $\nabla a$ exists a.e.,
while Assumption~\ref{ass:rkhs}(K2) and
Assumption~\ref{ass:stab}(S1) ensure that the first and second
derivatives of $k$ and $U_0$ respectively exist.   Per Step 1, we can differentiate \eqref{eq:disc:kl:pull} and get
\begin{equation}
\label{eq:disc:et:prime}
\begin{aligned}
E'(t)
&=
-\int_{(\mathbb R^d)^N}
\operatorname{tr}\!\bigl(A_t(x)^{-1}JF(x)\bigr)
g_K(\ell,x) dx
-\int_{(\mathbb R^d)^N}
\bigl\langle
\nabla\widetilde u_N(\phi_t(x)),F(x)
\bigr\rangle
g_K(\ell,x) dx.
\end{aligned}
\end{equation}
Now let us write down $E'(0)$. Since $\phi_0(x)=x$, $J\phi_0(x)=I$, we arrive at
\begin{align*}
E'(0)
&=
\int_{(\mathbb R^d)^N}
\Bigl(
-\mathrm{div}F(x)
-
\bigl\langle \nabla \widetilde u_N(x), F(x)\bigr\rangle
\Bigr) g_K(\ell,x) dx.
\end{align*}
Via \eqref{def:disc:drift}, we further obtain
\begin{equation}
\label{eq:disc:e0:prime}
E'(0)
=
\frac1N\sum_{k,j=1}^N
\int_{(\mathbb R^d)^N} g_K(\ell,x)
\Bigl(
-\mathrm{div}_{x_k}\zeta_a(x_k,x_j)
-\nabla \tilde u(x_k)\cdot \zeta_a(x_k,x_j)
\Bigr) dx.
\end{equation}

\smallskip
\underline{Step 2a:} We will handle the off-diagonal terms in \eqref{eq:disc:e0:prime}. Fix $k\neq j$ and write $x:=x_k$, $y:=x_j$. Recall $\zeta_a$ at \eqref{eq:zeta:a}. Since $\nabla_x k(y,x)=\nabla_2 k(y,x)$ and
$\div_x(\nabla_1 k(y,x))=\nabla_2\cdot\nabla_1 k(y,x)=\nabla_1\cdot\nabla_2 k(y,x)$,
we obtain
\begin{align}
\label{eq:disc:off:div}
-\div_x\zeta_{a}(x,y)
&=-\div_x\!\Big(a(x)a(y)\nabla_1 k(y,x)\Big)
  +\div_x\!\Big(a(x)a(y)k(y,x)\nabla U_0(y)\Big)\nonumber\\
&=-a(y)\nabla a(x)\cdot\nabla_1 k(y,x)
  -a(x)a(y) \div_x\nabla_1 k(y,x)\nonumber\\
&\quad
+a(y)\nabla a(x)\cdot\big(k(y,x)\nabla U_0(y)\big)
+a(x)a(y) \nabla_x k(y,x)\cdot\nabla U_0(y)\nonumber\\
&=a(y)\nabla a(x)\cdot\big(k(y,x)\nabla U_0(y)-\nabla_1 k(y,x)\big)\nonumber\\
&\quad
+a(x)a(y)\Big(\nabla_2 k(y,x)\cdot\nabla U_0(y)-\nabla_1\cdot\nabla_2 k(y,x)\Big).
\end{align}
Next,
\begin{align}
\label{eq:disc:off:nabla:util}
-\nabla\tilde u(x)\cdot\zeta_{a}(x,y)
&=-a(x)a(y)\nabla\tilde u(x)\cdot\Big(\nabla_1 k(y,x)-k(y,x)\nabla U_0(y)\Big)\nonumber\\
&=a(x)a(y)\Big(-\nabla\tilde u(x)\cdot\nabla_1 k(y,x)+k(y,x) \nabla\tilde u(x)\cdot\nabla U_0(y)\Big).
\end{align}
Adding \eqref{eq:disc:off:div} and \eqref{eq:disc:off:nabla:util} gives
for every $y$ and a.e.\ $x$,
\begin{align}
\label{eq:disc:off}
-\div_x\zeta_{a}(x,y)-\nabla\tilde u(x)\cdot\zeta_{a}(x,y)
&=a(y)\nabla a(x)\cdot\big(k(y,x)\nabla U_0(y)-\nabla_1 k(y,x)\big)\nonumber\\
&\quad
+a(x)a(y)\Big(
-\nabla_1\cdot\nabla_2 k(y,x)
+\nabla_2 k(y,x)\cdot\nabla U_0(y)\nonumber\\
&\hspace{1cm}
-\nabla\tilde u(x)\cdot\nabla_1 k(y,x)
+k(y,x) \nabla\tilde u(x)\cdot\nabla U_0(y)
\Big).
\end{align}
Recalling $u_a$ defined at \eqref{def:u:a} and combining this with \eqref{eq:disc:off}, we arrive at
\begin{equation}
\label{eq:anch:disc:off}
  -\div_x\zeta_a(x,y)-\nabla\tilde u(x)\cdot\zeta_a(x,y)
  =
  -u_a(y,x)\;+\;R_1(x,y)\;+\;R_2(x,y),
  \qquad (k\neq j),
\end{equation}
where $w_{\tilde p},R_1,R_2$ are defined in the statement.


\smallskip
\underline{Step 2b:} Next, we handle the diagonal terms in \eqref{eq:disc:e0:prime}. Fix $k\in\{1,\dots,N\}$ and set $x:=x_k$. Recall $\zeta_a$ at \eqref{eq:zeta:a}. Since $\div_x(\nabla_1 k(x,x))=\Delta_1 k(x,x)+\nabla_1\cdot\nabla_2 k(x,x)$, it follows that 
\begin{align}
\label{eq:disc:diag:div:fin}
-\div_x\zeta_a(x,x)
&=2a(x)\nabla a(x)\cdot\big(k(x,x)\nabla U_0(x)-\nabla_1 k(x,x)\big)\nonumber\\
&\quad
-a(x)^2\Big(\Delta_1 k(x,x)+\nabla_1\cdot\nabla_2 k(x,x)\Big)
+a(x)^2 \div_x\big(k(x,x)\nabla U_0(x)\big);
\\
\label{eq:disc:diag:nabla:util}
-\nabla\tilde u(x)\cdot\zeta_a(x,x)
&=a(x)^2\Big(-\nabla\tilde u(x)\cdot\nabla_1 k(x,x)
+k(x,x) \nabla\tilde u(x)\cdot\nabla U_0(x)\Big). 
\end{align}

Adding \eqref{eq:disc:diag:div:fin} and \eqref{eq:disc:diag:nabla:util} yields, a.e.\ $x$,
\begin{align}
\label{eq:disc:diag:lhs}
-\div_x\zeta_a(x,x)-\nabla\tilde u(x)\cdot\zeta_a(x,x)
&=2a(x)\nabla a(x)\cdot\big(k(x,x)\nabla U_0(x)-\nabla_1 k(x,x)\big)\nonumber\\
&\quad
+a(x)^2\Big(
-\Delta_1 k(x,x)-\nabla_1\cdot\nabla_2 k(x,x)
+\div_x\big(k(x,x)\nabla U_0(x)\big)\nonumber\\
&\hspace{2.7cm}
-\nabla\tilde u(x)\cdot\nabla_1 k(x,x)
+k(x,x) \nabla\tilde u(x)\cdot\nabla U_0(x)
\Big).
\end{align}

Recall $u_a$ at \eqref{def:u:a}, and $w_{\tilde p},R_1(x,x),R_2(x,x)$ in the statement. Then per \eqref{eq:disc:diag:lhs}, we can write
\begin{equation}
\label{eq:anch:disc:diag}
  -\div_x\zeta_a(x,x)-\nabla\tilde u(x)\cdot\zeta_a(x,x)
  =
  -u_a(x,x)\;+\;R_1(x,x)\;+\;R_2(x,x)\;+\; D_a(x),
\end{equation}
a.e., where $D_a$ is defined in the statement.

\smallskip
\underline{Step 2c:}
Recall the definition of KSD at \eqref{eq:anch:ksd} which is $\KSD_a^2\bigl(\mu_\ell^N\|p\bigr)
=
\frac1{N^2}\sum_{k,j=1}^N u_a(x_k,x_j)$. Combining it with \eqref{eq:anch:disc:off}, \eqref{eq:anch:disc:diag},
and \eqref{eq:disc:e0:prime}, we obtain
\begin{equation}
\label{eq:disc:e0:prime:decom}
E'(0)
=
-N\int_{(\mathbb R^d)^N} g_K(\ell,x)
\KSD_a^2\bigl(\mu_\ell^N\|p\bigr) dx
+R_3^{\mathrm{disc}}(\ell)+R_{\mathrm{diag}}^{\mathrm{disc}}(\ell)
\end{equation}
where $R_3^{\mathrm{disc}}(\ell)$ and $R_{\mathrm{diag}}^{\mathrm{disc}}(\ell)$ are
\begin{equation}
\label{def:disc:r3}
R_3^{\mathrm{disc}}(\ell)
:=
\frac1N\sum_{k,j=1}^N
\int_{(\mathbb R^d)^N} g_K(\ell,x) \Bigl(R_1(x_k,x_j)+R_2(x_k,x_j)\Bigr) dx
\end{equation}
and
\begin{equation}
\label{def:disc:rdiag}
R_{\mathrm{diag}}^{\mathrm{disc}}(\ell)
:=
\frac1N\sum_{k=1}^N
\int_{(\mathbb R^d)^N} g_K(\ell,x) D_a(x_k) dx.
\end{equation}


\medskip
\underline{Step 3:} Now we derive $E''(t)$. Starting from \eqref{eq:disc:et:prime}, write
\begin{equation}
\label{eq:anch:disc:e:prime:split}
E'(t)=I_1(t)+I_2(t),
\end{equation}
where
\begin{align}
I_1(t)
&:=
-\int_{(\mathbb R^d)^N}
\mathrm{tr}\!\bigl((J\phi_t(x))^{-1}JF(x)\bigr) g_K(\ell,x) dx,
\label{eq:disc:i1}\\
I_2(t)
&:=
-\int_{(\mathbb R^d)^N}
\bigl\langle \nabla \widetilde u_N(\phi_t(x)), F(x)\bigr\rangle g_K(\ell,x) dx.
\label{eq:disc:i2}
\end{align}

\smallskip
\underline{Step 3a:}
By the chain rule and \eqref{eq:disc:part:phi}, we get $\frac{d}{dt}\nabla \widetilde u_N(\phi_t(x))
=
\nabla^2 \widetilde u_N(\phi_t(x)) F(x)$, so that 
\begin{align}
\label{eq:disc:term:in:i2:prime}
\frac{d}{dt}\Bigl[-\bigl\langle \nabla \widetilde u_N(\phi_t(x)), F(x)\bigr\rangle\Bigr]
&=
-\Bigl\langle \frac{d}{dt}\nabla \widetilde u_N(\phi_t(x)), F(x)\Bigr\rangle =
-\bigl\langle F(x), \nabla^2 \widetilde u_N(\phi_t(x)) F(x)\bigr\rangle.
\end{align}
Consequently,
\begin{equation}
\label{eq:disc:i2:prime}
\psi_2(t):=I_2'(t)
=
-\int_{(\mathbb R^d)^N}
\bigl\langle F(x), \nabla^2 \widetilde u_N(\phi_t(x)) F(x)\bigr\rangle
 g_K(\ell,x) dx.
\end{equation}

\smallskip
\underline{Step 3b:}
Taking the derivative of $(J\phi_t(x))^{-1}J\phi_t(x)=I$ gives $\frac{d}{dt}(J\phi_t(x))^{-1}J\phi_t(x)
+
(J\phi_t(x))^{-1}\partial_tJ\phi_t(x)
=0$, 
which together with \eqref{eq:disc:part:phi} leads to
\begin{equation}
\label{eq:disc:der:inv:jphi}
\frac{d}{dt}(J\phi_t(x))^{-1}
=
-(J\phi_t(x))^{-1}JF(x)(J\phi_t(x))^{-1}.
\end{equation}
Now differentiate the integrand in \eqref{eq:disc:i1} of $I_1(t)$ and apply \eqref{eq:disc:der:inv:jphi} to get
\begin{align*}
\frac{d}{dt}\Bigl[-\mathrm{tr}\!\bigl((J\phi_t(x))^{-1}JF(x)\bigr)\Bigr]
&=
-\mathrm{tr}\!\Bigl(\frac{d}{dt}(J\phi_t(x))^{-1}JF(x)\Bigr) \notag\\
&\stackrel{\eqref{eq:disc:der:inv:jphi}}=
\mathrm{tr}\!\Bigl((J\phi_t(x))^{-1}JF(x)(J\phi_t(x))^{-1}JF(x)\Bigr).
\end{align*}
Therefore,
\begin{equation*}
I_1'(t)
=
\int_{(\mathbb R^d)^N}
\mathrm{tr}\!\Bigl((J\phi_t(x))^{-1}JF(x)(J\phi_t(x))^{-1}JF(x)\Bigr)
 g_K(\ell,x) dx.
\end{equation*}
Setting
\begin{equation}
\label{def:disc:bt}
B_t(x):=JF(x)(J\phi_t(x))^{-1},
\end{equation}
we have $\mathrm{tr}\!\Bigl((J\phi_t(x))^{-1}JF(x)(J\phi_t(x))^{-1}JF(x)\Bigr)
=
\mathrm{tr}\!\bigl(B_t(x)^2\bigr).$ Hence
\begin{align}
\label{eq:disc:i1:prime}
\psi_1(t)
&:=
I_1'(t)
=
\int_{(\mathbb R^d)^N}
\mathrm{tr}\!\bigl(B_t(x)^2\bigr)
g_K(\ell,x) dx\le
\int_{(\mathbb R^d)^N}
\|B_t(x)\|_{\mathrm{HS}}^2
g_K(\ell,x) dx,
\end{align}
where $|\mathrm{tr}(B_t(x)^2)|\le\|B_t(x)\|_{\mathrm{HS}}^2$.

\smallskip
\underline{Step 3c:}
Combining \eqref{eq:anch:disc:e:prime:split}, \eqref{eq:disc:i2:prime} and
\eqref{eq:disc:i1:prime} gives
\begin{align}
\label{eq:disc:e2der}
E''(t)
=
\psi_2(t)+\psi_1(t)&\stackrel{\eqref{eq:disc:i2:prime},\eqref{eq:disc:i1:prime}}{=}
-\int_{(\mathbb R^d)^N}
\bigl\langle F(x), \nabla^2 \widetilde u_N(\phi_t(x)) F(x)\bigr\rangle
 g_K(\ell,x) dx \notag\\
&\qquad
+\int_{(\mathbb R^d)^N}
\mathrm{tr}\!\bigl(B_t(x)^2\bigr)
g_K(\ell,x) dx\notag\\
&\stackrel{\eqref{eq:disc:i1:prime}}{\le}
-\int_{(\mathbb R^d)^N}
\bigl\langle F(x), \nabla^2 \widetilde u_N(\phi_t(x)) F(x)\bigr\rangle
g_K(\ell,x) dx\notag\\
&\qquad
+\int_{(\mathbb R^d)^N}
\|B_t(x)\|_{\mathrm{HS}}^2
g_K(\ell,x) dx.
\end{align}

\medskip
\noindent\underline{Step 4:}
Using \eqref{eq:disc:end:law} and $E(\eta)
=
E(0)+\eta E'(0)
+\int_0^\eta(\eta-t)E''(t) dt$, we get 
\begin{align*}
\mathrm{KL}\bigl(g_K(\ell+1)\|\tilde p^{\otimes N}\bigr)
&=
\mathrm{KL}\bigl(g_K(\ell)\|\tilde p^{\otimes N}\bigr)
+\eta E'(0)
+\int_0^\eta (\eta-t)E''(t) dt.
\end{align*}
Plugging in \eqref{eq:disc:e0:prime:decom} for $E'(0)$, we obtain
\begin{align}
\label{eq:disc:desc}
\mathrm{KL}\bigl(g_K(\ell+1)\|\tilde p^{\otimes N}\bigr)
&=
\mathrm{KL}\bigl(g_K(\ell)\|\tilde p^{\otimes N}\bigr)
-
N\eta \mathbb E_{g_K(\ell)}
\Bigl[
\KSD_a^2\bigl(\mu_\ell^N\|p\bigr)
\Bigr]\notag\\
&\qquad
+\eta\Bigl(
R_3^{\mathrm{disc}}(\ell)
+
R_{\mathrm{diag}}^{\mathrm{disc}}(\ell)
\Bigr)
+\int_0^\eta (\eta-t)E''(t) dt
\end{align}
and consequently the desired bound \eqref{eq:disc:ksd2e}. 
\end{proof}


\subsection{Proof of Theorem~\ref{theo:disc:bd:ksd2}}

Set 
\[Q:=p_0^{\otimes N},\qquad U_{0,N}(x):=\sum_{i=1}^N U_0(x_i),\qquad m:=dN.\]
We  will use
Proposition~\ref{prop:disc:bd:ksd:squar:all:norm} with
$\widetilde p=p_0$ and $\varepsilon_r=\eta$ for
$r=0,\dots,L-1$. 

\noindent\underline{Step 1:} We will verify the assumptions of Proposition~\ref{prop:disc:bd:ksd:squar:all:norm}. 

\begin{enumerate}[label=(\roman*),leftmargin=2.4em]
\item (A2), (K2) and (S1) as assumed in Theorem~\ref{theo:disc:bd:ksd2} are among those required in Proposition~\ref{prop:disc:bd:ksd:squar:all:norm}. 

\item $\mu_0(G_K)>0$ per \eqref{eq:mu0:gk:cond}, so that $g_K(0)$ and the truncated density in \eqref{def:disc:init} are well-defined.

\item Regarding \eqref{assum:algo:parti},  $\widetilde u=\log p_0=-U_0-\log Z_0\in C^2(\mathbb R^d)$ by
Assumption~\ref{ass:stab}(S1), while Equation~\eqref{eq:disc:f:loc:w1inf}
gives
$F\in W_{\mathrm{loc}}^{1,\infty}
((\mathbb R^d)^N;(\mathbb R^d)^N)$. Now we check that $\mathrm{KL}(g_K(0)\|Q)<\infty$. Via
\eqref{def:gk:trunc:dens}, we can write
\begin{align}
\label{eq:disc:kl:trunc:expan:anch}
\mathrm{KL}\bigl(g_K(0)\|Q\bigr)
&=
\int_{G_K}\frac{g(0,x)}{\mu_0(G_K)}
\Bigl[
\log\!\Bigl(\frac{g(0,x)}{Q(x)}\Bigr)
-
\log \mu_0(G_K)
\Bigr] dx
\notag\\
&=
\frac1{\mu_0(G_K)}
\int_{G_K} g(0,x)\log\!\Bigl(\frac{g(0,x)}{Q(x)}\Bigr) dx
+
\log\!\Bigl(\frac1{\mu_0(G_K)}\Bigr).
\end{align}
Set $r:=g(0)/Q$. Since $r\log r\ge-1/e$,
\[
\int_{G_K}g(0,x)\log\!\Bigl(\frac{g(0,x)}{Q(x)}\Bigr)dx
=
\mathrm{KL}\bigl(g(0)\|Q\bigr)
-
\int_{G_K^c}r\log r dQ
\le
\mathrm{KL}\bigl(g(0)\|Q\bigr)+\frac1e.
\]
Using $\mu_0(G_K)\ge\frac12$ and
$\mathrm{KL}(g(0)\|Q)\le\gamma dN$ from
\eqref{eq:mu0:gk:cond}, we obtain
\begin{equation}
\label{eq:disc:kl:trunc:bd:anch}
\mathrm{KL}\bigl(g_K(0)\|Q\bigr)
\le
2\mathrm{KL}\bigl(g(0)\|Q\bigr)+\frac2e+\log2
\stackrel{\eqref{eq:mu0:gk:cond}}{\le}
2\gamma dN+\frac2e+\log2.
\end{equation}

\item
Here we verify \eqref{assum:pos:det}. \eqref{eq:disc:jphi:inv:bd:anch} gives
$\|tJF(x)\|_{\mathrm{op}}\le\frac12$ for every $t\in[0,\eta]$.
Therefore $A_t(x)=I+tJF(x)$ is invertible and $\det A_t(x)\neq 0$ for all $t$. The map
$t\mapsto\det A_t(x)$ is continuous so that $\det A_t(x)>0$ for every $t$. Indeed, suppose $\det A_{t_0}(x)<0$ for some $t_0$; then continuity of $A_t$, $\det A_0(x)=1$ and the intermediate value theorem imply $\det A_{t_1}=0$ for $0\leq t_1\leq t_0$, which is a contradiction.

\item

Next is \eqref{assum:pos:det}. By \eqref{def:disc:gk:ell}, $g_K(\ell)$ is the law of $x^\ell$ when
$x^0$ has density $g_K(0)$. Since $g_K(0)$ is supported on $G_K$,
\eqref{eq:disc:sum:u0:apr:bd} provides for $g_K(\ell)$-a.e.\ $x$
\begin{align}
\label{eq:bound:u0:n}
\frac1N U_{0,N}(x)
\le
M\Bigl(d^{\frac{1}{1-\alpha_0}}+K\Bigr)
\Bigl((\eta L)^{\frac{1}{1-\alpha_0}}\vee1\Bigr)
=:H.
\end{align}
Equations \eqref{eq:disc:f:bd:sk} and \eqref{eq:disc:jf:bd:sk} give
constants $C_F,C_J<\infty$ such that $|F(x)|\le C_F$ and
$\|JF(x)\|_{\mathrm{HS}}\le C_J$ for $g_K(\ell)$-a.e.\ $x$.
Moreover, \eqref{eq:disc:jphi:inv:bd:anch} and
\eqref{eq:disc:jphi:inv:op:bd:anch} give
$\|tJF(x)\|_{\mathrm{op}}\le\frac12$ and
$\|A_t(x)^{-1}\|_{\mathrm{op}}\le2$ for $0\le t\le\eta$.
The singular values of $A_t(x)$ therefore lie in
$[\frac12,\frac32]$. Since $\det A_t(x)>0$ as shown above and $|\det A_t(x)|$ is the product of the singular values of $A_t(x)$,  we get $2^{-m}\le\det A_t(x)\le\left(\frac32\right)^m$ which confirms \eqref{assum:pos:det}. In addition, 
\begin{align}
\label{eq:log:det:at}
|\log\det A_t(x)|\le m\log2.
\end{align}

\item We will verify \eqref{ass:disc:diff:env}. 
Taylor's theorem and
$\|\nabla^2U_0\|_{\mathrm{op}}\le C_{U_0}$ from Assumption~\ref{ass:stab}(S1) give
\[
U_{0,N}(\phi_t(x))
\le
U_{0,N}(x)
+t|\nabla U_{0,N}(x)||F(x)|
+\frac{C_{U_0}t^2}{2}|F(x)|^2.
\]
Moreover,  Assumption~\ref{ass:stab}(S1), Jensen's inequality and \eqref{eq:bound:u0:n} give the estimate $|\nabla U_{0,N}(x)|^2
\\\le
A_0^2N
\left(\frac1N U_{0,N}(x)\right)^{2\alpha_0}\le
A_0^2NH^{2\alpha_0}$ a.e., which combined with the above bound for $U_{0,N}(\phi_t(x))$ leads to
\[
U_{0,N}(\phi_t(x))
\le
NH+\eta A_0\sqrt N H^{\alpha_0}C_F
+\frac{C_{U_0}\eta^2}{2}C_F^2
=:C_{U_0}^{\mathrm{disc}}.
\]
Since $\widetilde u_N=-U_{0,N}-N\log Z_0$,
$|\widetilde u_N(\phi_t(x))|
\le C_{U_0}^{\mathrm{disc}}+N|\log Z_0|$. Again by
Assumption~\ref{ass:stab}(S1) and Jensen's inequality,
\[
|\nabla\widetilde u_N(\phi_t(x))|^2
\le
A_0^2N
\left(\frac{U_{0,N}(\phi_t(x))}{N}\right)^{2\alpha_0}
\le
A_0^2N
\left(\frac{C_{U_0}^{\mathrm{disc}}}{N}\right)^{2\alpha_0}.
\]
The Hessian of $\widetilde u_N= -\sum_{i=1}^NU_0(x_i)-N\log Z_0$ is a block diagonal matrix with $\nabla^2\widetilde u_N(x) = \operatorname{diag}\bigl( -\nabla^2U_0(x_1),\ldots,-\nabla^2U_0(x_N) \bigr)$, so
$\|\nabla^2\widetilde u_N(x)\|_{\mathrm{op}} = \max_{1\le i\le N}\|\nabla^2U_0(x_i)\|_{\mathrm{op}} \le C_{U_0}$ per Assumption~\ref{ass:stab}(S1). Using these bounds together with
$|\log\det A_t(x)|\le m\log2$ from \eqref{eq:log:det:at},
$\|A_t(x)^{-1}\|_{\mathrm{op}}\le2$ from
\eqref{eq:disc:jphi:inv:op:bd:anch},
$|F(x)|\le C_F$ from \eqref{eq:disc:f:bd:sk}, and
$\|JF(x)\|_{\mathrm{HS}}\le C_J$ from
\eqref{eq:disc:jf:bd:sk}, we deduce for $D_\ell$ at \eqref{def:dell}, it holds $g_K(\ell)$-a.e. that 
\begin{align*}
\mathcal D_\ell(x)
\le&
m\log2+C_{U_0}^{\mathrm{disc}}+N|\log Z_0|+2C_J\\
&+
A_0\sqrt N
\left(\frac{C_{U_0}^{\mathrm{disc}}}{N}\right)^{\alpha_0}C_F
+4C_J^2+C_{U_0}C_F^2<\infty.
\end{align*}
Since
$g_K(\ell)$ is a probability density, we get $\int_{(\mathbb R^d)^N}
g_K(\ell,x)\mathcal D_\ell(x)dx<\infty$ as the first part in \eqref{ass:disc:diff:env}. 

The last step is to check
$\int g_K(\ell,x)|\log g_K(\ell,x)|dx<\infty$ as the second part of
\eqref{ass:disc:diff:env}. The fact that
$\sup_{0<s\le1}s|\log s|=e^{-1}$ gives
\[
\int_{\{g_K(0)\le Q\}}
g_K(0,x)
\left|\log\frac{g_K(0,x)}{Q(x)}\right|dx
\le
e^{-1}\int Q(x)dx
=
e^{-1}.
\]
Thus the negative part of
$g_K(0)\log(g_K(0)/Q)$ is integrable. Also, 
\[\mathrm{KL}(g_K(0)\|Q) = \int f(x)dx = \int \brac{g_K(0,x)\log\frac{g_K(0,x)}{Q(x)}}^+dx-\int \brac{g_K(0,x)\log\frac{g_K(0,x)}{Q(x)}}^-dx\]
which is finite per \eqref{eq:disc:kl:trunc:bd:anch}. Then the positive part is also integrable: $\int g_K(0,x)|\log(g_K(0,x)/Q(x))|dx<\infty$. Since $g_K(0)$ is
supported on $G_K$ and $|\log Q(x)|\le NK+N|\log Z_0|$ on $G_K$,
$\int g_K(0,x)|\log Q(x)|dx<\infty$. Furthermore, with
$\log g_K(0)=\log(g_K(0)/Q)+\log Q$ on $\{g_K(0)>0\}$, it follows that
\begin{align}\label{eq:int:gk:log:gk}\int g_K(0,x)|\log g_K(0,x)|dx<\infty.\end{align}

At this point, change of variables and
$g_K(\ell+1,T_\ell(x))\det A_\eta(x)=g_K(\ell,x)$ (similar to \eqref{eq:push:expl:gen}) give
\begin{align*}
\int g_K(\ell+1,y)|\log g_K(\ell+1,y)|dy
&=
\int g_K(\ell,x)
\left|\log g_K(\ell,x)-\log\det A_\eta(x)\right|dx\\
&\stackrel{\eqref{eq:log:det:at}}{\le}
\int g_K(\ell,x)|\log g_K(\ell,x)|dx+dN\log2,
\end{align*}
where the last inequality uses \eqref{eq:log:det:at} with $m=dN$.
Repeating this argument inductively and using \eqref{eq:int:gk:log:gk} imply
$\int g_K(\ell,x)|\log g_K(\ell,x)|dx<\infty$ for
$0\le\ell\le L-1$.

\end{enumerate}

\smallskip
\noindent\underline{Step 2:} Let us apply Proposition~\ref{prop:disc:bd:ksd:squar:all:norm}. 
Equation~\eqref{eq:disc:rem:bd:anch} gives
$\int_0^\eta(\eta-t)E''(t)dt\le\theta^2$ for
$0\le\ell\le L-1$ and $\theta=\sqrt{N/L}$ as chosen in \eqref{eq:disc:unnor:t:theta}. Since $\widetilde p=p_0$, we have the cancellation
\[w_{\widetilde p}=\nabla\log p_0+\nabla U_0\equiv0.\]
Therefore,
\eqref{eq:disc:ksd2e} gives
\begin{align}
\label{eq:disc:avg:bd:pre}
\mathbb E_{g_K(0)}
\Biggl[
\frac1L\sum_{\ell=0}^{L-1}\KSD_a^2\bigl(\mu_\ell^N\|p\bigr)
\Biggr]
&\le
\frac{\mathrm{KL}\bigl(g_K(0)\|p_0^{\otimes N}\bigr)}{NL\eta}
\notag\\
&\qquad
+\frac{1}{NL}\sum_{\ell=0}^{L-1}
\Bigl(
R_3^{\mathrm{disc}}(\ell)
+
R_{\mathrm{diag}}^{\mathrm{disc}}(\ell)
\Bigr)
+
\frac{\theta^2}{N\eta}.
\end{align}
By \eqref{eq:disc:kl:trunc:bd:anch}, the KL term in
\eqref{eq:disc:avg:bd:pre} is at most
$(2\gamma d+2/e+\log2)/(L\eta)$. Substituting this KL bound into
\eqref{eq:disc:avg:bd:pre}, we obtain
\begin{align}
\label{eq:disc:avg:bd:anch}
&\mathbb E_{g_K(0)}
\Biggl[
\frac1L\sum_{\ell=0}^{L-1}\KSD_a^2\bigl(\mu_\ell^N\|p\bigr)
\Biggr]\notag\\
&\le
\left(2\gamma d+\frac2e+\log2\right)\frac{1}{L\eta}
+\frac{1}{NL}\sum_{\ell=0}^{L-1}
\Bigl(
R_3^{\mathrm{disc}}(\ell)
+
R_{\mathrm{diag}}^{\mathrm{disc}}(\ell)
\Bigr)
+
\frac{\theta^2}{N\eta}.
\end{align}
Now let $A:=N^{1-\alpha_0}L^{2\alpha_0}$.
By \eqref{eq:disc:unnor:t:theta}, $L\ge\overline L$ and
$2\alpha_0-1<0$, so $\frac{A}{L/N}
=N^{2-\alpha_0}L^{2\alpha_0-1}
\le
N^{2-\alpha_0}\overline L^{2\alpha_0-1}
=1$. 
Moreover, $A\ge
N^{1-\alpha_0}\overline L^{2\alpha_0}
=\frac{\overline L}{N}
\ge N$. Thus $A\vee N=A$. Since $\theta=\sqrt{N/L}$,
\eqref{eq:disc:unnor:eta} gives
\[
\frac{1}{L\eta}
=
\sqrt{C_0}\sqrt{\frac{A}{NL}}
\le
\frac{\sqrt{C_0}}{N},
\qquad
\frac{\theta^2}{N\eta}
=
\sqrt{C_0}\sqrt{\frac{A}{NL}}
\le
\frac{\sqrt{C_0}}{N}.
\]
Then \eqref{eq:disc:fin:anch} follows from
\eqref{eq:disc:avg:bd:anch}.

\smallskip
\underline{Step 3:} Finally, we handle the remainder terms in \eqref{eq:disc:avg:bd:anch}. We start with bounding $R_3^{\mathrm{disc}}(\ell)$, noting that the choice
$\tilde{p}:=p_0$ means $w_{\tilde{p}}\equiv 0$.
Using $|a|\le A_\infty$ and $|\nabla a|\le L_a$ from (A2) of
Assumption~\ref{ass:bl:ext}, $|k|\le B$, and
$|\nabla_1 k(y,x)|\le B\sqrt d$ from (K2) of
Assumption~\ref{ass:rkhs}, we obtain
\begin{align*}
|R_2(x,y)|
&\le
|a(y)| |\nabla a(x)|
\Bigl(
|k(y,x)| |\nabla U_0(y)|+|\nabla_1 k(y,x)|
\Bigr)\\
&\le
A_\infty L_a B |\nabla U_0(y)|
+
A_\infty L_a B \sqrt d.
\end{align*}
Therefore,
\begin{align*}
|R_3^{\mathrm{disc}}(\ell)|
&\le
\frac1N\sum_{k,j=1}^N
\int g_K(\ell,x) |R_2(x_k,x_j)| dx\\
&\le
N A_\infty L_a B
\mathbb E_{g_K(\ell)}
\Biggl[
\frac1N\sum_{j=1}^N |\nabla U_0(x_j)|
\Biggr]
+
N A_\infty L_a B \sqrt d.
\end{align*}
Set
\begin{align}
\label{def:hl}
h(\ell):=\frac1N\sum_{i=1}^N U_0\bigl(x_i^\ell\bigr),
\qquad \ell=0,1,\dots,L-1.
\end{align}
By the growth assumption in Assumption~\ref{ass:stab}(S1),
$|\nabla U_0(x)|\le A_0 U_0(x)^{\alpha_0}$ and Jensen's inequality,
\begin{align}
\label{eq:disc:bd:sum:grad:u}
\frac1N\sum_{j=1}^N |\nabla U_0(x_j)|
\le
A_0\frac1N\sum_{j=1}^N U_0(x_j)^{\alpha_0}
\le
A_0\Bigl(\frac1N\sum_{j=1}^N U_0(x_j)\Bigr)^{\alpha_0}
=
A_0 h(\ell)^{\alpha_0}.
\end{align}
Hence
\begin{equation}
\label{eq:disc:r3pre}
|R_3^{\mathrm{disc}}(\ell)|
\le
N A_\infty L_a B d
+
N A_\infty L_a B A_0
\mathbb E_{g_K(\ell)}\bigl[h(\ell)^{\alpha_0}\bigr].
\end{equation}
Now via Lemma~\ref{lem:disc:bd:sum:u0}(ii) and the truncated
initialization $g_K(0)$ which implies $h(0)\le K$ a.s., we get
\[
h(\ell)^{\alpha_0}
\le
M^{\alpha_0}
\Bigl(d^{\frac{1}{1-\alpha_0}}+K\Bigr)^{\alpha_0}
\Bigl((\eta \ell)^{\frac{\alpha_0}{1-\alpha_0}}\vee 1\Bigr)
\qquad g_K(0)\text{-a.s.}
\]
Substituting this into \eqref{eq:disc:r3pre} gives
\eqref{eq:disc:r3bd}. We next bound
$R_{\mathrm{diag}}^{\mathrm{disc}}(\ell)$. Recall the function $D_a$ defined at \eqref{def:w:rem:da}. 
The fact that $\div_x\bigl(k(x,x)\nabla U_0(x)\bigr)
=
\bigl(\nabla_1 k(x,x)+\nabla_2 k(x,x)\bigr)
\cdot\nabla U_0(x)
+k(x,x)\Delta U_0(x)$ implies 
\[
\div_x\bigl(k(x,x)\nabla U_0(x)\bigr)
-\nabla U_0(x)\cdot\nabla_2 k(x,x)
=
\nabla_1 k(x,x)\cdot\nabla U_0(x)+k(x,x)\Delta U_0(x),
\]
so that
\begin{align*}
D_a(x)
&=
a(x)\nabla a(x)\cdot
\bigl(k(x,x)\nabla U_0(x)-\nabla_1 k(x,x)\bigr)\\
&\quad
+a(x)^2\Bigl(
\nabla_1 k(x,x)\cdot\nabla U_0(x)
+k(x,x)\Delta U_0(x)
-\Delta_1 k(x,x)
\Bigr).
\end{align*}
Therefore, using $|a|\le A_\infty$ and
$|\nabla a|\le L_a$ from Assumption~\ref{ass:bl:ext}(A2),
$|k(x,x)|\le B$, $|\nabla_1 k(x,x)|\le B\sqrt d$, and
$|\Delta_1 k(x,x)|\le Bd$ from (K2) of
Assumption~\ref{ass:rkhs}, and \eqref{eq:disc:lap:u0:bd},
\begin{align*}
|D_a(x)|
&\le
|a(x)| |\nabla a(x)|
\Bigl(
|k(x,x)| |\nabla U_0(x)|+|\nabla_1 k(x,x)|
\Bigr)\\
&\quad
+
|a(x)|^2
\Bigl(
|\nabla_1 k(x,x)| |\nabla U_0(x)|
+|k(x,x)| |\Delta U_0(x)|
+|\Delta_1 k(x,x)|
\Bigr)\\
&\stackrel{\eqref{eq:disc:lap:u0:bd}}{\le}
A_\infty L_a\bigl(B|\nabla U_0(x)|+B\sqrt d\bigr)
+
A_\infty^2\bigl(
B\sqrt d |\nabla U_0(x)|+BC_\Delta d+Bd
\bigr)\\
&\le
B\bigl(A_\infty L_a+A_\infty^2\bigr)
\sqrt d |\nabla U_0(x)|
+
B\bigl(
A_\infty L_a+A_\infty^2+A_\infty^2C_\Delta
\bigr)d.
\end{align*}
Consequently,
\begin{align*}
|R_{\mathrm{diag}}^{\mathrm{disc}}(\ell)|
&\le
\mathbb E_{g_K(\ell)}
\Biggl[
\frac1N\sum_{k=1}^N |D_a(x_k)|
\Biggr]\le
\beta_{d,1} d
+
\beta_{d,2}\sqrt d
\mathbb E_{g_K(\ell)}
\Biggl[
\frac1N\sum_{k=1}^N |\nabla U_0(x_k)|
\Biggr],
\end{align*}
where
$\frac1N\sum_{k=1}^N |\nabla U_0(x_k)|
\le A_0 h(\ell)^{\alpha_0}$ per
\eqref{eq:disc:bd:sum:grad:u}.
Using the bound on $h(\ell)^{\alpha_0}$ obtained above from
Lemma~\ref{lem:disc:bd:sum:u0}(ii) and $h(0)\le K$, we obtain
\eqref{eq:disc:rdiag:bd}. This completes the proof.
\QEDA


\bibliographystyle{amsplain}

\begin{thebibliography}{99}


\bibitem{ledoux2023gaussian}
M.~Ledoux,
\emph{Gaussian concentration inequalities},
The Gaussian Blog, December 2023.

\bibitem{ambrosio2004transport}
L.~Ambrosio.
\newblock Transport equation and Cauchy problem for BV vector fields.
\newblock  {Inventiones Mathematicae}, 158(2):227-260, 2004.


\bibitem{Evans}
L.~C.~Evans.
\newblock \emph{Partial Differential Equations}.
\newblock Graduate Studies in Mathematics, vol.~19.
\newblock American Mathematical Society, Providence, RI, second edition, 2010.



\bibitem{ambrosioGigliSavare2005}
L.~Ambrosio, N.~Gigli, and G.~Savar\'e.
\newblock \emph{Gradient Flows in Metric Spaces and in the Space of
Probability Measures}.
\newblock Lectures in Mathematics ETH Z\"urich.
Birkh\"auser, Basel, 2005.


\bibitem{vanNeerven1992}
J.~M.~A.~M. van Neerven.
\newblock \emph{The Adjoint of a Semigroup of Linear Operators}.
\newblock Lecture Notes in Mathematics, vol.~1529.
\newblock Springer-Verlag, Berlin, 1992.


\bibitem{banerjeeetal2025}
S.~Banerjee, K.~Balasubramanian, and P.~Ghosal.
\newblock Improved finite-particle convergence rates for Stein variational gradient descent.
\newblock In \emph{The Thirteenth International Conference on Learning Representations}, 2025.


\bibitem{hytonenEtAl2016}
T.~Hytönen, J.~van Neerven, M.~Veraar, and L.~Weis,
\emph{Analysis in Banach Spaces. Volume I: Martingales and
Littlewood-Paley Theory},
Ergebnisse der Mathematik und ihrer Grenzgebiete, 3. Folge, vol.~63,
Springer, Cham, 2016.


\bibitem{LLN}
Jianfeng Lu, Yulong Lu, and James Nolen.
\newblock Scaling limit of the Stein variational gradient descent: the mean field regime.
\newblock \emph{SIAM Journal on Mathematical Analysis}, 51(2):648\textendash671, 2019.


\bibitem{LLJ}
Q.~Liu, J.~D.~Lee, and M.~I.~Jordan.
\newblock A kernelized Stein discrepancy for goodness-of-fit tests.
\newblock In \emph{Proceedings of the 33rd International Conference on
Machine Learning}, volume~48 of \emph{Proceedings of Machine Learning
Research}, pages 276\textendash284, 2016.

\bibitem{OGC}
C.~J.~Oates, M.~Girolami, and N.~Chopin.
\newblock Control functionals for Monte Carlo integration.
\newblock \emph{Journal of the Royal Statistical Society: Series B
(Statistical Methodology)}, 79(3):695\textendash718, 2017.


\bibitem{liu2017flow}
Q. Liu.
Stein variational gradient descent as gradient flow.
In \textit{Advances in Neural Information Processing Systems},
volume 30, 2017.

\bibitem{CSG}
K.~Chwialkowski, H.~Strathmann, and A.~Gretton.
\newblock A kernel test of goodness of fit.
\newblock In \emph{Proceedings of the 33rd International Conference on
Machine Learning}, volume~48 of \emph{Proceedings of Machine Learning
Research}, pages 2606\textendash2615, 2016.

\bibitem{LiuWang}
Q.~Liu and D.~Wang.
\newblock Stein variational gradient descent: A general purpose Bayesian
inference algorithm.
\newblock In \emph{Advances in Neural Information Processing Systems 29},
pages 2378\textendash2386, 2016.

\bibitem{HanLiu}
J.~Han and Q.~Liu.
\newblock Stein variational gradient descent without gradient.
\newblock In \emph{Proceedings of the 35th International Conference on
Machine Learning}, volume~80 of \emph{Proceedings of Machine Learning
Research}, pages 1900\textendash1908, 2018.

\bibitem{federer}
H.~Federer.
\newblock \emph{Geometric Measure Theory}.
\newblock Grundlehren der mathematischen Wissenschaften, volume~153.
Springer, New York, 1969.


\bibitem{villani2003topics}
C.~Villani.
\newblock \emph{Topics in Optimal Transportation}.
\newblock Graduate Studies in Mathematics, volume~58.
American Mathematical Society, Providence, RI, 2003.

\bibitem{golse}
F.~Golse.
\newblock On the dynamics of large particle systems in the mean field limit.
\newblock In A.~Muntean, J.~Rademacher, and A.~Zagaris, editors,
\emph{Macroscopic and Large Scale Phenomena: Coarse Graining, Mean Field
Limits and Ergodicity}, volume~3 of \emph{Lecture Notes in Applied
Mathematics and Mechanics}, pages 1\textendash144.
Springer, Cham, 2016.



\bibitem{shi2022mirrored}
J. Shi, C. Liu and L. Mackey.
Sampling with mirrored Stein operators.
In \textit{International Conference on Learning Representations}, 2022.

\bibitem{liu2018riemannian}
C. Liu and J. Zhu.
Riemannian Stein variational gradient descent for Bayesian inference.
In \textit{Proceedings of the AAAI Conference on Artificial Intelligence},
volume 32, 2018.

\bibitem{banales2025branching}
I. Ba\~nales, A. Jaramillo and J. H. Ricalde-Guerrero.
Branching Stein variational gradient descent for sampling multimodal
distributions.
\textit{arXiv preprint arXiv:2506.13916}, 2025.

\bibitem{detommaso2018newton}
G. Detommaso, T. Cui, A. Spantini, Y. Marzouk and R. Scheichl.
A Stein variational Newton method.
In \textit{Advances in Neural Information Processing Systems},
volume 31, 2018.

\bibitem{wang2019matrix}
D. Wang, Z. Tang, C. Bajaj and Q. Liu.
Stein variational gradient descent with matrix-valued kernels.
In \textit{Advances in Neural Information Processing Systems},
volume 32, 2019.

\bibitem{chen2020projected}
P. Chen and O. Ghattas.
Projected Stein variational gradient descent.
In \textit{Advances in Neural Information Processing Systems},
volume 33, 2020.

\bibitem{gong2021sliced}
W. Gong, Y. Li and J. M. Hern\'andez-Lobato.
Sliced kernelized Stein discrepancy.
In \textit{International Conference on Learning Representations}, 2021.

\bibitem{liu2022grassmann}
X. Liu, H. Zhu, J.-F. Ton, G. Wynne and A. Duncan.
Grassmann Stein variational gradient descent.
In \textit{Proceedings of the 25th International Conference on
Artificial Intelligence and Statistics},
volume 151 of \textit{Proceedings of Machine Learning Research},
pages 2002\textendash2021, 2022.

\bibitem{wang2018message}
D. Wang, Z. Zeng and Q. Liu.
Stein variational message passing for continuous graphical models.
In \textit{Proceedings of the 35th International Conference on
Machine Learning},
volume 80 of \textit{Proceedings of Machine Learning Research},
pages 5219\textendash5227, 2018.

\bibitem{li2020stochastic}
L. Li, Y. Li, J.-G. Liu, Z. Liu and J. Lu.
A stochastic version of Stein variational gradient descent for
efficient sampling.
\textit{Communications in Applied Mathematics and Computational Science},
15:37\textendash63, 2020.

\bibitem{feng2017amortized}
Y. Feng, D. Wang and Q. Liu.
Learning to draw samples with amortized Stein variational gradient descent.
In \textit{Proceedings of the 33rd Conference on Uncertainty in
Artificial Intelligence}, 2017.

\bibitem{han2017importance}
J. Han and Q. Liu.
Stein variational adaptive importance sampling.
In \textit{Proceedings of the 33rd Conference on Uncertainty in
Artificial Intelligence}, 2017.

\bibitem{sharrock2023coin}
L. Sharrock and C. Nemeth.
Coin sampling: Gradient-based Bayesian inference without learning rates.
In \textit{Proceedings of the 40th International Conference on
Machine Learning},
volume 202 of \textit{Proceedings of Machine Learning Research},
pages 30850\textendash30882, 2023.



\bibitem{han2020discrete}
J. Han, F. Ding, X. Liu, L. Torresani, J. Peng and Q. Liu.
Stein variational inference for discrete distributions.
In \textit{Proceedings of the Twenty Third International Conference on
Artificial Intelligence and Statistics},
volume 108 of \textit{Proceedings of Machine Learning Research},
pages 4563\textendash4572, 2020.



\bibitem{lovasz2007geometry}
L.~Lov\'asz and S.~Vempala.
The geometry of logconcave functions and sampling algorithms.
\textit{Random Structures \& Algorithms}, 30:307–358, 2007.

\bibitem{lovasz2006fast}
L.~Lov\'asz and S.~Vempala.
Fast algorithms for logconcave functions: Sampling, rounding,
integration and optimization.
In \textit{Proceedings of the 47th Annual IEEE Symposium on
Foundations of Computer Science}, pp.~57–68, 2006.

\bibitem{pereyra2016proximal}
M.~Pereyra.
Proximal Markov chain Monte Carlo algorithms.
\textit{Statistics and Computing}, 26:745–760, 2016.

\bibitem{durmus2018moreau}
A.~Durmus, \'E.~Moulines and M.~Pereyra.
Efficient Bayesian computation by proximal Markov chain Monte Carlo:
When Langevin meets Moreau.
\textit{SIAM Journal on Imaging Sciences}, 11(1):473–506, 2018.

\bibitem{salim2019splitting}
A.~Salim, D.~Kovalev and P.~Richt\'arik.
Stochastic proximal Langevin algorithm: Potential splitting and
nonasymptotic rates.
In \textit{Advances in Neural Information Processing Systems},
volume~32, 2019.

\bibitem{chatterji2020smoothness}
N.~Chatterji, J.~Diakonikolas, M.~I.~Jordan and P.~Bartlett.
Langevin Monte Carlo without smoothness.
In \textit{Proceedings of the 23rd International Conference on
Artificial Intelligence and Statistics},
volume~108 of \textit{Proceedings of Machine Learning Research},
pp.~1716–1726, 2020.

\bibitem{nakakita2026langevin}
S.~Nakakita.
Non-asymptotic analysis of Langevin-type Monte Carlo algorithms.
\textit{Advances in Applied Probability}, 58(3):918–955, 2026.

\bibitem{habring2024subgradient}
A.~Habring, M.~Holler and T.~Pock.
Subgradient Langevin methods for sampling from nonsmooth potentials.
\textit{SIAM Journal on Mathematics of Data Science},
6(4):897–925, 2024.

\bibitem{gurbuzbalaban2025anchored}
M.~G\"urb\"uzbalaban, H.~M.~Nguyen, X.~Zhang and L.~Zhu.
Anchored Langevin algorithms.
arXiv:2509.19455, 2025.

\bibitem{lee2021structured}
Y.~T.~Lee, R.~Shen and K.~Tian.
Structured logconcave sampling with a restricted Gaussian oracle.
In \textit{Proceedings of the 34th Conference on Learning Theory},
volume~134 of \textit{Proceedings of Machine Learning Research},
pp.~2993–3050, 2021.

\bibitem{chen2022proximal}
Y.~Chen, S.~Chewi, A.~Salim and A.~Wibisono.
Improved analysis for a proximal algorithm for sampling.
In \textit{Proceedings of the 35th Conference on Learning Theory},
volume~178 of \textit{Proceedings of Machine Learning Research},
pp.~2984–3014, 2022.

\bibitem{liang2022proximal}
J.~Liang and Y.~Chen.
A proximal algorithm for sampling.
arXiv:2202.13975, 2022, revised 2023.

\bibitem{fan2023dimension}
J.~Fan, B.~Yuan and Y.~Chen.
Improved dimension dependence of a proximal algorithm for sampling.
In \textit{Proceedings of the 36th Conference on Learning Theory},
volume~195 of \textit{Proceedings of Machine Learning Research},
pp.~1473–1521, 2023.

\bibitem{kook2026zeroth}
Y.~Kook and S.~S.~Vempala.
Zeroth-order logconcave sampling.
arXiv:2507.18021, version~2, 2026.

\bibitem{montanari2023sampling}
A.~Montanari.
Sampling, diffusions, and stochastic localization.
arXiv:2305.10690, 2023.

\bibitem{he2024zeroth}
Y.~He, K.~Rojas and M.~Tao.
Zeroth-order sampling methods for non-log-concave distributions:
Alleviating metastability by denoising diffusion.
In \textit{Advances in Neural Information Processing Systems},
volume~37, 2024.

\bibitem{han2025splitting}
F.~Han, S.~Osher and W.~Li.
Splitting regularized Wasserstein proximal algorithms for
nonsmooth sampling problems.
arXiv:2502.16773, 2025.

\bibitem{cordero2025proximal}
P.~Cordero Encinar, F.~R.~Crucinio and O.~D.~Akyildiz.
Proximal interacting particle Langevin algorithms.
In \textit{Proceedings of the 41st Conference on Uncertainty
in Artificial Intelligence},
volume~286 of \textit{Proceedings of Machine Learning Research},
pp.~1220–1265, 2025.


\end{thebibliography}


\newpage
\appendix
\section{Sampling targets with non-smooth potentials}

Sampling from distributions with nonsmooth potentials on $\R^d$
has been studied through several approaches. Geometric random walks,
including the ball walk and hit-and-run, require no derivatives of
the potential \cite{lovasz2007geometry,lovasz2006fast}. Meanwhile, Langevin methods
have been extended to non-smooth sampling problems via proximal updates and potential splitting
\cite{pereyra2016proximal,durmus2018moreau,salim2019splitting},
Gaussian and spherical smoothing
\cite{chatterji2020smoothness,nakakita2026langevin},
subgradient updates \cite{habring2024subgradient}, and anchoring
\cite{gurbuzbalaban2025anchored}. Another approach introduces auxiliary
variables and alternates conditional sampling steps, as in proximal
samplers
\cite{lee2021structured,chen2022proximal,liang2022proximal,
fan2023dimension,kook2026zeroth}.
A group of related constructions which use increasingly informative observations
connects sampling with stochastic localization and denoising diffusion
\cite{montanari2023sampling,he2024zeroth}. Finally,
interacting particle methods have also been developed for nonsmooth
sampling through Wasserstein proximal splitting
\cite{han2025splitting} and proximal particle Langevin algorithms
\cite{cordero2025proximal}.

\section{Proof and supporting lemmas for Section~\ref{sec:anch:ksd}}
\label{appen:anch:ksd}

\textbf{Proof of Part i of Theorem~\ref{theo:anch:ksd:form:zero}:}
By Theorem~\ref{theo:opt:dir}, $\KSD_a(q\|p)=\|h\|_{\mathcal H^d}$ where $h_i$ is provided in \eqref{def:opt:dir:comp}. Hence
\begin{align}
\label{eq:anch:ksd:norm}
\KSD_a^2(q\|p)=\|h\|_{\mathcal H^d}^2=\sum_{i=1}^d\langle h_i,h_i\rangle_{\mathcal H}.
\end{align}
Expanding the square, applying Fubini's theorem, and using the reproducing property yield
\begin{align*}
\KSD_a^2(q\|p)
&=
\sum_{i=1}^d
\mathbb E_{(X,Y)\sim q\otimes q}\Big[
a(X)a(Y)
\big\langle
\partial_{x_i}k(X,\cdot)-(\partial_iU_0)(X)k(X,\cdot),\\
&\hspace{4.1cm}
\partial_{x_i}k(Y,\cdot)-(\partial_iU_0)(Y)k(Y,\cdot)
\big\rangle_{\mathcal H}
\Big] \\
&=
\mathbb E_{(X,Y)\sim q\otimes q}\big[u_a(X,Y)\big],
\end{align*}
which is \eqref{eq:anch:ksd}.
\QEDA

\textbf{Proof of Part ii of Theorem~\ref{theo:anch:ksd:form:zero}:}
By \eqref{eq:anch:ksd:orig}, the condition
$\mathrm{KSD}_a(q\|p)=0$ means
\[
\mathbb E_q\bigl[a(X)A_{p_0}^{\top}\phi(X)\bigr]=0
\qquad
\forall \phi\in\mathcal H^d.
\]
Fix $j\in\{1,\dots,d\}$ and $x\in\mathbb R^d$, and choose $\phi(\cdot)=e_j k(x,\cdot)\in \mathcal H^d$. 
Then by Lemma~\ref{lem:repr:impl},
\[
0
=
\mathbb E_q\bigl[a(X)A_{p_0}^{\top}(e_jk(x,\cdot))(X)\bigr]
=
\Bigl\langle
\mathbb E_q\bigl[\xi_{p_0}^a(X,\cdot)\bigr],
 e_j k(x,\cdot)
\Bigr\rangle_{\mathcal H^d}.
\]
Since $h_j(x)=\langle h,e_j k(x,\cdot)\rangle_{\mathcal H^d}=0$ for all $x$ and $j$, each component of $h:=\mathbb E_q\bigl[\xi_{p_0}^a(X,\cdot)\bigr]\in \mathcal H^d$ vanishes pointwise on $\mathbb R^d$. 
Therefore,
\begin{equation}
\label{eq:mean0:p0}
\mathbb E_q\bigl[\xi_{p_0}^a(X,\cdot)\bigr]=0
\qquad\text{in }\mathcal H^d.
\end{equation}
Moreover, by Lemma~\ref{lem:q0xi0},
\begin{equation}
\label{eq:mean:q:xi0}
\mathbb E_q\bigl[\xi_{q_0}^a(X,\cdot)\bigr]=0
\qquad\text{in }\mathcal H^d.
\end{equation}
Subtracting \eqref{eq:mean:q:xi0} from
\eqref{eq:mean0:p0} and using
\eqref{eq:comp:meas:stein}, we obtain
\begin{align}
\label{eq:mean0:anch:rat:k}
\mathbb E_q\Bigl[
a(X)\nabla\log\frac{p(X)}{q(X)} k(X,\cdot)
\Bigr]=0
\qquad\text{in }\mathcal H^d.
\end{align}

Define
\[
h(x):=a(x)\nabla\log\frac{p(x)}{q(x)}.
\]
Then $h\in L^1(q;\mathbb R^d)$ by (Q2) in
Assumption~\ref{ass:q0:kern}. In view of \eqref{eq:mean0:anch:rat:k}, taking inner product of $h$
with arbitrary $\phi\in\mathcal H^d$ and using the reproducing property yield $\int_{\mathbb R^d} h(x)\cdot \phi(x) q(x) dx=0$ for all $\phi\in\mathcal H^d$.  Then per Lemma~\ref{lem:c0:univ} with $\nu(dx)=q(x) dx$,
we conclude that $h(x)=0$ a.e. $x$. Since $a(x)=e^{U(x)-U_0(x)}>0$ for every $x$, it follows that $\nabla\log\frac{p(x)}{q(x)}=0$ a.e. $x$. Because $q$ is a strictly positive density, $\nabla\log\frac{p}{q}=0$ a.e. Now since $\log\frac{p}{q}\in \mathrm{Lip}_{\mathrm{loc}}(\mathbb R^d)$, we have
\[
\log\frac{p}{q}=C
\qquad
\text{a.e. on }\mathbb R^d
\]
for some constant $C\in\mathbb R$. Since both $p$ and $q$ are probability densities, we conclude $p=q$ a.e.
\QEDA

We recall the following identity  from \cite{HanLiu}. 
\begin{lemma}
\label{lem:relat:anch:sop}
Recall the standard Stein operator
\begin{align*}
A_p^\top\phi(x):=\nabla\log p(x)\cdot \phi(x)+\nabla\cdot \phi(x). 
\end{align*}
Then under Assumptions~\ref{ass:bl:ext} and~\ref{ass:rkhs}, it holds that 
\begin{equation*}
a(x)A_{p_0}^{\top}\phi(x)=A_p^{\top}(a\phi)(x) \qquad\text{a.e. $x$.}
\end{equation*}
\end{lemma}

\begin{proof}
Since $a(x)p(x)=\frac{Z_0}{Z} p_0(x)$ a.e., we have $\nabla\log p_0(x)=\nabla\log p(x)+\nabla\log a(x)$ a.e. Then by the product rule,
\begin{align}
\label{eq:conv:anch:op}
a(x) A_{p_0}^\top\phi(x)
&=a(x)\big(\nabla\log p_0(x)\cdot \phi(x)+\nabla\cdot\phi(x)\big) \nonumber\\
&=a(x)\big(\nabla\log p(x)\cdot\phi(x)+\nabla\cdot\phi(x)\big)+a(x)\nabla\log a(x)\cdot\phi(x) \nonumber\\
&=a(x)\nabla\log p(x)\cdot\phi(x)+\nabla\cdot(a(x)\phi(x)) \nonumber\\
&=A_p^\top\big(a\phi\big)(x),
\end{align}
which completes the proof. 
\end{proof}

For any density $\rho$ on $\mathbb R^d$ with a.e.\ defined gradient
$\nabla\log \rho$, define
\begin{equation}
\label{def:xia}
\xi_{\rho}^a(x,\cdot)
:=
a(x)\Bigl(\nabla\log \rho(x) k(x,\cdot)+\nabla_x k(x,\cdot)\Bigr)
\in \mathcal H^d.
\end{equation}

We record a simple consequence of the reproducing property.
\begin{lemma}
\label{lem:repr:impl}
Assume Assumption~\ref{ass:rkhs}. Let $\rho$ be a density such that
$a(x)\nabla\log \rho(x)$ is defined for a.e.\ $x$. Then for every
$\phi\in\mathcal H^d$ and a.e.\ $x\in\mathbb R^d$,
\begin{equation*}
\bigl\langle \phi,\xi_{\rho}^a(x,\cdot)\bigr\rangle_{\mathcal H^d}
=
a(x)A_{\rho}^{\top}\phi(x).
\end{equation*}
\end{lemma}

\begin{proof}
Write $\phi=(\phi_1,\dots,\phi_d)$. By
Lemma~\ref{lem:k2:k3:der:repr}, $\partial_i\phi_i(x)
=
\bigl\langle \phi_i,\partial_{x_i}k(x,\cdot)\bigr\rangle_{\mathcal H}$, so that 
\begin{align*}
\bigl\langle \phi,\xi_{\rho}^a(x,\cdot)\bigr\rangle_{\mathcal H^d}
&=
a(x)\sum_{i=1}^d
\Bigl\langle
\phi_i,
(\partial_i\log\rho(x)) k(x,\cdot)+\partial_{x_i}k(x,\cdot)
\Bigr\rangle_{\mathcal H} \\
&=
a(x)\sum_{i=1}^d
\Bigl(
(\partial_i\log\rho(x)) \phi_i(x)+\partial_i\phi_i(x)
\Bigr) \\
&=
a(x)\bigl(\nabla\log\rho(x)\cdot\phi(x)+\div\phi(x)\bigr)=
a(x)A_{\rho}^{\top}\phi(x),
\end{align*}
which completes the proof. 
\end{proof}


\begin{lemma}
\label{lem:q0xi0}
Assume Assumptions~\ref{ass:rkhs} and~\ref{ass:q0:kern}, and in addition either Assumption~\blue{M.2} or Assumption~\blue{M.2'}. Then for $\xi^a_{q_0}$ defined at \eqref{def:xia}, 
\begin{equation}
\label{eq:q0:xi:zero}
\mathbb E_q\bigl[\xi_{q_0}^a(X,\cdot)\bigr]=0
\qquad\text{in }\mathcal H^d.
\end{equation}
\end{lemma}

\begin{proof}
Under Assumption~\blue{M.2'}, fix $x\in\mathbb R^d$ and
$j\in\{1,\dots,d\}$. The weak product rule gives
\[
\partial_{y_j}\!\bigl(q_0(y)k(x,y)\bigr)
=
(\partial_jq_0)(y)k(x,y)+q_0(y)\partial_{y_j}k(x,y)
\qquad\text{a.e. }y.
\]
Since $q_0( \cdot )k(x, \cdot )\in W^{1,1}(\mathbb R^d)$,
the integral over $\mathbb R^d$ of this weak derivative is zero. Since
$q_0$ is strictly positive,
\[
\partial_jq_0=q_0 \partial_j\log q_0
\qquad\text{a.e.},
\]
and the corresponding terms are integrable by (Q2). Substitution in the
previous identity gives \eqref{eq:q0:kern:stein}. Hence
Assumption~\blue{M.2'} implies
Assumption~\blue{M.2}, and it suffices to do the proof under
Assumption~\blue{M.2}.

Set
\[
c_q:=\int_{\mathbb R^d} a(y)q(y) dy.
\]
Fix $j\in\{1,\dots,d\}$ and $x\in\mathbb R^d$. By the reproducing property,
\begin{align*}
\Bigl\langle
\mathbb E_q\bigl[\xi_{q_0}^a(X,\cdot)\bigr],
 e_j k(x,\cdot)
\Bigr\rangle_{\mathcal H^d}
&=
\mathbb E_q\!\Big[
a(X)\Big((\partial_j\log q_0)(X) k(x,X)+\partial_{x_j}k(X,x)\Big)
\Big] \\
&=
c_q \mathbb E_{q_0}\!\Big[
(\partial_j\log q_0)(X) k(x,X)+\partial_{x_j}k(X,x)
\Big].
\end{align*}
Since $k$ is symmetric, $\partial_{x_j}k(X,x)=\partial_{y_j}k(x,X)$, which, combined with \eqref{eq:q0:kern:stein}, gives
\[
\Bigl\langle
\mathbb E_q\bigl[\xi_{q_0}^a(X,\cdot)\bigr],
 e_j k(x,\cdot)
\Bigr\rangle_{\mathcal H^d}
=0
\qquad
\forall x,\ \forall j \qquad\text{where}\quad e_j=(0,\dots,0,1,0,\dots,0). 
\]
Therefore each component of $\mathbb E_q\bigl[\xi_{q_0}^a(X,\cdot)\bigr]\in\mathcal H^d$
vanishes pointwise on $\mathbb R^d$ and we conclude 
 $\mathbb E_q\bigl[\xi_{q_0}^a(X,\cdot)\bigr]=0$. 
\end{proof}

\begin{lemma}
\label{lem:anch:feat:comp}
Under Assumption~\ref{ass:q0:kern}, the following identity holds in
$\mathcal H^d$:
\begin{equation}
\label{eq:comp:meas:stein}
\mathbb E_q\bigl[\xi_{p_0}^a(X,\cdot)\bigr]
=
\mathbb E_q\bigl[\xi_{q_0}^a(X,\cdot)\bigr]
+
\mathbb E_q\Bigl[
a(X)\nabla\log\frac{p(X)}{q(X)} k(X,\cdot)
\Bigr].
\end{equation}
\end{lemma}

\begin{proof}
By definition,
\[
\xi_{p_0}^a(x,\cdot)-\xi_{q_0}^a(x,\cdot)
=
a(x)\bigl(\nabla\log p_0(x)-\nabla\log q_0(x)\bigr) k(x,\cdot).
\]
Since $\log p_0(x)=\log p(x)+\log a(x)+\log\Bigl(\frac{Z}{Z_0}\Bigr)$, we have $\nabla\log p_0(x)=\nabla\log p(x)+\nabla\log a(x)$ a.e.
Similarly, we have $\nabla\log q_0(x)=\nabla\log q(x)+\nabla\log a(x)$ a.e. Thus, 
\[
\nabla\log p_0(x)-\nabla\log q_0(x)
=
\nabla\log\frac{p(x)}{q(x)}
\qquad\text{a.e.}
\]
Thus, for a.e. $x$,
\[
\xi_{p_0}^a(x,\cdot)-\xi_{q_0}^a(x,\cdot)
=
a(x)\nabla\log\frac{p(x)}{q(x)} k(x,\cdot)
\qquad\text{in }\mathcal H^d.
\]
Taking
expectations yields \eqref{eq:comp:meas:stein}.
\end{proof}

The following states a simple consequence of $C_0$-universality. 
\begin{lemma}
\label{lem:c0:univ}
Assume Assumption~\ref{ass:q0:kern}(Q1). Let $\nu$ be a finite
Borel measure on $\mathbb R^d$ and let $h\in L^1(\nu;\mathbb R^d)$.
If
\begin{equation}
\label{eq:c0:anni}
\int_{\mathbb R^d} h(x)\cdot \phi(x) \nu(dx)=0
\qquad
\forall \phi\in\mathcal H^d,
\end{equation}
then $h=0$ a.e. with respect to $\nu$. 
\end{lemma}

\begin{proof}
Define a vector-valued finite signed measure $\mu=(\mu_1,\dots,\mu_d)$ by
\[
\mu_i(B):=\int_B h_i(x) \nu(dx),
\qquad
B\subset\mathbb R^d \text{ Borel}.
\]
Then \eqref{eq:c0:anni} says
\[
\sum_{i=1}^d \int_{\mathbb R^d}\phi_i(x) \mu_i(dx)=0
\qquad
\forall \phi=(\phi_1,\dots,\phi_d)\in\mathcal H^d.
\]
Since $\mathcal H$ is dense in $C_0(\mathbb R^d)$ in the uniform norm,
$\mathcal H^d$ is dense in $C_0(\mathbb R^d;\mathbb R^d)$ in the uniform norm.
Therefore the above identity extends by continuity to all
$C_0(\mathbb R^d;\mathbb R^d)$.
Choosing test functions with only one nonzero component gives
\[
\int_{\mathbb R^d} f(x) \mu_i(dx)=0
\qquad
\forall f\in C_0(\mathbb R^d),\ \forall i.
\]
Hence $\mu_i=0$ for every $i$, and thus $h_i=0$ a.e. with respect to $\nu$.
\end{proof}

\section{Proof for Section~\ref{sec:ref:pot}}
\label{appen:ref:pot}

\textbf{Proof of Lemma~\ref{lem:u0:moll}:}

Since $U$ is locally Lipschitz, there exists a measurable set $E\subset\mathbb R^d$ of full Lebesgue measure such that at every $x\in E$, $\nabla U$ exists and 
\begin{equation}
\label{eq:e:good:set}
|\nabla U(x)-\nabla U(y)|\le L_1|x-y|+L_2
\qquad\text{for a.e. }y\in\mathbb R^d.
\end{equation}

\underline{Verifying Assumption~\ref{ass:bl:ext}(A1):} \eqref{eq:e:good:set} clearly implies $|\nabla U(x)|
\le C_U\bigl(1+|x|\bigr)$ and 
$|U(x)|
\le C_U\bigl(1+|x|^2\bigr)$ a.e. $x$ in (A1).  

Next, by \eqref{eq:const:u0:mom:cond}, we have $\int_{\mathbb R^d}(1+|z|^2)|\partial^\alpha\rho(z)| dz<\infty$ for $|\alpha|\le2$. Combined with the quadratic growth of $U$, this ensures that, for every $x\in\mathbb R^d$ and $|\alpha|\le2$, $\int_{\mathbb R^d}
|U(y)| |\partial^\alpha\rho(x-y)| dy<\infty$. The dominated convergence theorem then yields for every $|\alpha|\le2$,
\[
\partial^\alpha U_0(x)
=
\int_{\mathbb R^d}U(y)\partial^\alpha\rho(x-y) dy
=
\int_{\mathbb R^d}U(x-z)\partial^\alpha\rho(z) dz
\]
Moreover, since $\partial_{x_i}\rho(x-y)
=
-\partial_{y_i}\rho(x-y)$, we get $\partial_i U_0(x)
=
-\int_{\mathbb R^d}U(y) \partial_{y_i}\rho(x-y) dy$. Then  we use a cutoff approximation, \eqref{eq:const:u0:mom:cond} and integration by parts to obtain
\begin{equation}
\label{eq:grad:u0}
\nabla U_0(x)
=
\int_{\mathbb R^d}\nabla U(x-z)\rho(z) dz
\qquad\text{for a.e. }x\in\mathbb R^d.
\end{equation}
Therefore, for a.e.\ $x\in\mathbb R^d$,
\[
|\nabla U_0(x)|
\le
\int_{\mathbb R^d}|\nabla U(x-z)|\rho(z) dz
\le
C\int_{\mathbb R^d}(1+|x-z|)\rho(z) dz
\le
C''(1+|x|)
\]
for some constant $C''<\infty$. Since $\nabla U_0$ is continuous, the same bound holds for every $x\in\mathbb R^d$. Thus, Assumption~\ref{ass:bl:ext}(A1) holds with $m=1$.

\underline{Verifying Assumption~\ref{ass:bl:ext}(A2):}
Since $U\in\mathrm{Lip}_{\mathrm{loc}}(\mathbb R^d)$, for a.e. $(x,z)\in\mathbb R^d\times\mathbb R^d$, the map $t\mapsto U(x-tz)$
is absolutely continuous on $[0,1]$ and satisfies $\frac{d}{dt}U(x-tz)
=
-\nabla U(x-tz)\cdot z$ a.e. $t\in(0,1)$. Consequently,
\[
U(x-z)-U(x)+\nabla U(x)\cdot z
=
\int_0^1
\bigl(\nabla U(x)-\nabla U(x-tz)\bigr)\cdot z dt.
\]
We also know  $|\nabla U(x)-\nabla U(x-tz)|
\le
L_1t|z|+L_2$ a.e. $(x,z,t)$, so plug this into the above equation to get for a.e. $(x,z)$,
\begin{equation}
\label{eq:u:taylo:rem}
|U(x-z)-U(x)+\nabla U(x)\cdot z|
\le
\frac{L_1}{2}|z|^2+L_2|z|.
\end{equation}
Via the assumption $\int_{\mathbb R^d}z\rho(z) dz=0$, we obtain for a.e. $x$,
\begin{align*}
|U_0(x)-U(x)|
&=
\left|
\int_{\mathbb R^d}
\bigl(U(x-z)-U(x)+\nabla U(x)\cdot z\bigr)
\rho(z) dz
\right|\\
&\stackrel{\eqref{eq:u:taylo:rem}}{\le}
B:=\frac{L_1}{2}\int_{\mathbb R^d}|z|^2\rho(z) dz
+
L_2\int_{\mathbb R^d}|z|\rho(z) dz.
\end{align*}
This implies the desired bound $\|a\|_{L^\infty(\mathbb R^d)}
\le e^{B}$. It also gives
\[
e^{-B}e^{-U(x)}
\le e^{-U_0(x)}
\le e^{B}e^{-U(x)},
\]
and hence
\[
0<\int_{\mathbb R^d}e^{-U_0(x)} dx<\infty.
\]

Moreover, by \eqref{eq:grad:u0}, for a.e. $x$,
\begin{align*}
|\nabla(U-U_0)(x)|
&\le
\int_{\mathbb R^d}
|\nabla U(x)-\nabla U(x-z)|
\rho(z) dz\stackrel{\eqref{eq:e:good:set}}{\le}
L_1\int_{\mathbb R^d}|z|\rho(z) dz+L_2.
\end{align*}
Combining with  $\nabla a=a\nabla(U-U_0)$ a.e.  to get  $\|\nabla a\|_{L^\infty(\mathbb R^d)}
\le
e^{B}
\left(
L_1\int_{\mathbb R^d}|z|\rho(z) dz+L_2
\right)$. Therefore $a\in W^{1,\infty}(\mathbb R^d)$.

\underline{Verifying Assumption~\ref{ass:stab}(S1):}
Since $\inf_{x\in\mathbb R^d}U(x)>0$, we have $U(x-z)\ge \inf_{y\in\mathbb R^d}U(y)$ a.e. $x,z$, which further implies 
\[
U_0(x)
=
\int_{\mathbb R^d}U(x-z)\rho(z) dz
\ge
\int_{\mathbb R^d}\inf_{y\in\mathbb R^d}U(y) \rho(z) dz
=\inf_{y\in\mathbb R^d}U(y)>0. 
\]

Next, using \eqref{eq:grad:u0}, \eqref{eq:uolem:add:ass}, and Jensen's inequality, for a.e.\ $x\in\mathbb R^d$,
\begin{align*}
|\nabla U_0(x)|
\stackrel{\eqref{eq:grad:u0}}{\le}
\int_{\mathbb R^d}|\nabla U(x-z)|\rho(z) dz
&\stackrel{\eqref{eq:uolem:add:ass}}{\le}
A_0\int_{\mathbb R^d}U(x-z)^{\alpha_0}\rho(z) dz\\
&\le
A_0\left(\int_{\mathbb R^d}U(x-z)\rho(z) dz\right)^{\alpha_0}
=
A_0U_0(x)^{\alpha_0}.
\end{align*}
Since both sides are continuous, this inequality holds for every $x\in\mathbb R^d$.

Finally, fix a multi-index $\alpha=(e_i,e_j)$ with $|\alpha|=e_i+e_j=2$. We claim that \begin{align}\label{eq:int:der:rho}
\int_{\mathbb R^d}\partial^\alpha\rho(z) dz=0,
\qquad
\int_{\mathbb R^d}z_i\partial^\alpha\rho(z) dz=0,
\quad i=1,\dots,d.
\end{align}
Indeed, with a smooth cutoff function $\chi_R$, $\int \chi_R \partial_{jk}\rho
=
-\int (\partial_j\chi_R) \partial_k\rho$
and the right-hand side tends to zero as $R\to\infty$ given that  $\partial_k\rho\in L^1.$ Moreover, the first part of \eqref{eq:int:der:rho} also implies $\int_{\mathbb R^d}z_i\partial_{jk}\rho(z) dz
=
-\int_{\mathbb R^d}\partial_j(z_i) \partial_k\rho(z) dz
=
-\delta_{ij}\int_{\mathbb R^d}\partial_k\rho(z) dz
=0$. Now combining \eqref{eq:grad:u0} and \eqref{eq:int:der:rho} to get for $x\in E$,
\[
\partial^\alpha U_0(x)
=
\int_{\mathbb R^d}
\bigl(U(x-z)-U(x)+\nabla U(x)\cdot z\bigr)
\partial^\alpha\rho(z) dz.
\]
Via \eqref{eq:u:taylo:rem},
\[
|\partial^\alpha U_0(x)|
\le
\frac{L_1}{2}\int_{\mathbb R^d}|z|^2|\partial^\alpha\rho(z)| dz
+
L_2\int_{\mathbb R^d}|z||\partial^\alpha\rho(z)| dz.
\]
Since $\partial^\alpha U_0$ is continuous and $E$ has full measure, this bound extends to every $x\in\mathbb R^d$. Hence $\|HU_0(x)\|_{\mathrm{op}}\le C_{U_0}$ for all $x$,  where  $C_{U_0}
:=d\max_{|\alpha|=2}
\left(
\frac{L_1}{2}\int_{\mathbb R^d}|z|^2|\partial^\alpha\rho(z)| dz
+
L_2\int_{\mathbb R^d}|z||\partial^\alpha\rho(z)| dz
\right).$
Thus, all the properties in Assumption~\ref{ass:stab}(S1) hold. 
\QEDA

\section{Supporting lemmas for Section \ref{sec:anch:svgd}}
\label{sec:disc:supp:lem}

\begin{lemma} 
\label{lem:disc:bd:sum:u0}
Assume Assumption~\ref{ass:bl:ext}(A2), Assumption~\ref{ass:rkhs}(K2), and Assumption~\ref{ass:stab}(S1).
\begin{enumerate}[label=(\roman*),leftmargin=2.4em]
\item
With respect to the particle drift term at \eqref{def:disc:drift}, it holds for every $x=(x_1,\dots,x_N)\in(\mathbb R^d)^N$ that 
\begin{equation}
\label{eq:disc:bd:norm:f}
|F(x)|^2
\le
2A_\infty^4A_0^2B^2N
\Bigl(\frac1N\sum_{i=1}^N U_0(x_i)\Bigr)^{2\alpha_0}
+
2A_\infty^4B^2Nd.
\end{equation}

\item With respect to the function $h(\ell)$ at \eqref{def:hl},  there exist positive constants $M,D$, depending only on $A_\infty,B,A_0,C_{U_0},\alpha_0$, such that for $L\ge 1$,
$0\le \ell\le L$, and
\[
\eta\le 1\wedge \frac{1}{\sqrt{D L}},
\]
we have
\begin{equation}
\label{eq:disc:sum:u0:apr:bd}
h(\ell)
\le
M\Bigl(d^{\frac{1}{1-\alpha_0}}+\frac1N\sum_{i=1}^N U_0\bigl(x_i^0\bigr)\Bigr)
\Bigl((\eta \ell)^{\frac{1}{1-\alpha_0}}\vee 1\Bigr).
\end{equation}

\item
Under the hypotheses of Part ii, let $M,D$ be the constants from Part ii, and let $C_a$ be defined in \eqref{eq:disc:param:defs}.
Then we have
\begin{align}
|F(x^\ell)|^2
&\le
2A_\infty^4A_0^2B^2N 
M^{2\alpha_0}\Bigl(d^{\frac{1}{1-\alpha_0}}+h(0)\Bigr)^{2\alpha_0}
\Bigl((\eta \ell)^{\frac{2\alpha_0}{1-\alpha_0}}\vee 1\Bigr)
+
2A_\infty^4B^2Nd \label{eq:disc:fz:bd},\\
\|JF(x^\ell)\|_{\mathrm{HS}}^2
&\le
4A_0^2B^2C_a 
Nd M^{2\alpha_0}\Bigl(d^{\frac{1}{1-\alpha_0}}+h(0)\Bigr)^{2\alpha_0}
\Bigl((\eta \ell)^{\frac{2\alpha_0}{1-\alpha_0}}\vee 1\Bigr)\nn
\\&\hspace{10em}+
4B^2\bigl(C_a+3A_\infty^4C_{U_0}^2\bigr)Nd^2.\label{eq:disc:jf:z:bd}
\end{align}
\item Fix $K>0$, and suppose $x^0\in G_K$ for $G_K$ defined in \eqref{def:gk:set}. Recall $\beta,\Gamma_K$, and $C_a$ defined in \eqref{eq:disc:param:defs}, while $M,D$ are the constants from Part ii.
Define $C_0$ by \eqref{eq:disc:c0anc}.
Choose
\begin{equation}
\label{eq:disc:unnor:eta:pre}
\eta
=
\left[
\frac{1}{C_0}
\left(
\frac{1}{N^{1-\alpha_0}L^{2\alpha_0}}
\wedge
\frac1N
\right)
\right]^{1/2}\theta,
\qquad
0\le\theta\le1,
\qquad
\theta^2\le\frac NL.
\end{equation}
Then for every $0\le \ell\le L$ and every $t\in[0,\eta]$, it holds for the quantity $E''$ in \eqref{eq:disc:e2der} that
\begin{equation}
\label{eq:disc:e2:bd:anch}
E''(t)\le \frac{5\theta^2}{4\eta^2},
\end{equation}
and consequently
\begin{equation}
\label{eq:disc:rem:bd:anch}
\int_0^\eta (\eta-t)E''(t) dt\le \theta^2.
\end{equation}
\end{enumerate}
\end{lemma}

\begin{proof}
\noindent\underline{Proof of Part i.}
By Jensen's inequality,
\begin{equation}
\label{eq:disc:jens}
|F_k(x)|^2
=
\Bigl|
\frac1N\sum_{j=1}^N \zeta_a(x_k,x_j)
\Bigr|^2
\le
\frac1N\sum_{j=1}^N |\zeta_a(x_k,x_j)|^2.
\end{equation}
Via $|a|\le A_\infty$ from Assumption~\ref{ass:bl:ext}(A2), we get
\begin{align}
\label{eq:anch:lem2:zeta:pre}
|\zeta_a(x_k,x_j)|^2
&\le
A_\infty^4
\Bigl|
\nabla_1 k(x_j,x_k)-k(x_j,x_k)\nabla U_0(x_j)
\Bigr|^2 \notag\\
&\le
2A_\infty^4|\nabla_1 k(x_j,x_k)|^2
+
2A_\infty^4|k(x_j,x_k)|^2 |\nabla U_0(x_j)|^2.
\end{align}

Using the bound $|\partial_{x_i}k(x,y)|\le B$ for each coordinate $i$ from Assumption~\ref{ass:rkhs}(K2),
\begin{equation}
\label{eq:disc:k:der:bd}
|\nabla_1 k(x_j,x_k)|^2
=
\sum_{i=1}^d |\partial_{x_i}k(x_j,x_k)|^2
\le
dB^2.
\end{equation}
Similarly, by Assumption~\ref{ass:rkhs}(K2),
\begin{equation}
\label{eq:disc:k:bd}
|k(x_j,x_k)|\le B.
\end{equation}
By the growth assumption $|\nabla U_0(x)|\le A_0 U_0(x)^{\alpha_0},\forall x$ from Assumption~\ref{ass:stab}(S1),
\begin{equation}
\label{eq:disc:u0:norm:bd}
|\nabla U_0(x_j)|^2
\le
A_0^2 U_0(x_j)^{2\alpha_0}.
\end{equation}
Substituting \eqref{eq:disc:k:der:bd}, \eqref{eq:disc:k:bd},
and \eqref{eq:disc:u0:norm:bd} into \eqref{eq:anch:lem2:zeta:pre}, we obtain
\begin{equation}
\label{eq:anch:lem2:zeta:bd}
|\zeta_a(x_k,x_j)|^2
\le
2A_\infty^4B^2d
+
2A_\infty^4A_0^2B^2 U_0(x_j)^{2\alpha_0}.
\end{equation}
Combining \eqref{eq:disc:jens} and \eqref{eq:anch:lem2:zeta:bd} gives
\begin{equation}
\label{eq:anch:lem2:fk:bd}
|F_k(x)|^2
\le
2A_\infty^4B^2d
+
2A_\infty^4A_0^2B^2 
\frac1N\sum_{j=1}^N U_0(x_j)^{2\alpha_0}.
\end{equation}
Since $2\alpha_0\le 1$ and $U_0>0$, Jensen's inequality yields
\begin{equation*}
\frac1N\sum_{j=1}^N U_0(x_j)^{2\alpha_0}
\le
\Bigl(\frac1N\sum_{j=1}^N U_0(x_j)\Bigr)^{2\alpha_0}.
\end{equation*}
Hence
\begin{equation*}
|F_k(x)|^2
\le
2A_\infty^4B^2d
+
2A_\infty^4A_0^2B^2
\Bigl(\frac1N\sum_{j=1}^N U_0(x_j)\Bigr)^{2\alpha_0}.
\end{equation*}
Finally, summing over $k=1,\dots,N$ gives the stated result in Part i.

\medskip
\noindent\underline{Proof of Part ii.}
Fix $\ell\ge 0$. By Taylor's theorem and the update $x_i^{\ell+1}=x_i^\ell+\eta F_i\bigl(x^\ell\bigr)$, 
we have
\begin{align*}
h(\ell+1)-h(\ell)
&=
\frac1N\sum_{i=1}^N
\Bigl(U_0\bigl(x_i^{\ell+1}\bigr)-U_0\bigl(x_i^\ell\bigr)\Bigr)\\
&=
\frac{\eta}{N}\sum_{i=1}^N
\bigl\langle \nabla U_0(x_i^\ell),F_i(x^\ell)\bigr\rangle\\
&\qquad
+\frac{\eta^2}{N}\sum_{i=1}^N\int_0^1
(1-s)
\bigl\langle F_i(x^\ell),
HU_0\bigl(x_i^\ell+s\eta F_i(x^\ell)\bigr)F_i(x^\ell)\bigr\rangle
 ds.
\end{align*}
Using $\|HU_0\|_{\mathrm{op}}\le C_{U_0}$ from Assumption~\ref{ass:stab}(S1), we get
\begin{equation}
\label{eq:disc:h:rec}
h(\ell+1)-h(\ell)
\le
\frac{\eta}{N}\sum_{i=1}^N
\bigl\langle \nabla U_0(x_i^\ell),F_i(x^\ell)\bigr\rangle
+\frac{C_{U_0}\eta^2}{2N}|F(x^\ell)|^2.
\end{equation}

By the definition of $F_i$ at \eqref{def:disc:drift},
\begin{align*}
\frac1N\sum_{i=1}^N \bigl\langle \nabla U_0(x_i),F_i(x)\bigr\rangle
&=
\frac1{N^2}\sum_{i,j=1}^N
a(x_i)a(x_j)
\Bigl(
\bigl\langle \nabla U_0(x_i),\nabla_1 k(x_j,x_i)\bigr\rangle
-k(x_j,x_i)\bigl\langle \nabla U_0(x_i),\nabla U_0(x_j)\bigr\rangle
\Bigr).
\end{align*}
Since $k$ is positive definite by definition and $a>0$, the second term is nonpositive, hence
\begin{align}
\label{eq:disc:bd:grad:u0:f}
\frac1N\sum_{i=1}^N \bigl\langle \nabla U_0(x_i),F_i(x)\bigr\rangle
&\le
\frac1{N^2}\sum_{i,j=1}^N
|a(x_i)| |a(x_j)| |\nabla U_0(x_i)| |\nabla_1 k(x_j,x_i)| \notag\\
&\le
A_\infty^2B\sqrt d 
\frac1N\sum_{i=1}^N |\nabla U_0(x_i)| \le
A_\infty^2BA_0\sqrt d 
\Bigl(\frac1N\sum_{i=1}^N U_0(x_i)\Bigr)^{\alpha_0}.
\end{align}
In the last step we used $|a|\le A_\infty$ from Assumption~\ref{ass:bl:ext}(A2), $|\nabla_1 k(x_j,x_i)|\le B\sqrt d$ from Assumption~\ref{ass:rkhs}(K2), and
$|\nabla U_0|\le A_0U_0^{\alpha_0}$ from Assumption~\ref{ass:stab}(S1), together with Jensen's inequality.

Combining \eqref{eq:disc:bd:norm:f}, \eqref{eq:disc:bd:grad:u0:f} and \eqref{eq:disc:h:rec},
and writing
\[
B_0:=A_\infty^2BA_0,
\qquad
D_1:=C_{U_0}A_\infty^4A_0^2B^2,
\qquad
D_2:=C_{U_0}A_\infty^4B^2,
\]
we obtain
\begin{equation}
\label{eq:anch:h:rec}
h(\ell+1)\le h(\ell)+B_0\sqrt d \eta h(\ell)^{\alpha_0}+D_1\eta^2 h(\ell)^{2\alpha_0}+D_2d \eta^2.
\end{equation}
By \eqref{def:hl} and Assumption~\ref{ass:stab}(S1),
$h(\ell)>0$. Since $2\alpha_0\le1$,
\[
h(\ell)^{2\alpha_0}\le 1+h(\ell),
\qquad h(\ell)\ge 0.
\]
Hence, setting $D_3:=D_1+D_2$,
\begin{equation}
\label{eq:anch:h:rec:fin}
h(\ell+1)\le (1+D_1\eta^2)h(\ell)+B_0\sqrt d \eta h(\ell)^{\alpha_0}+D_3d \eta^2.
\end{equation}
From this point, we follow the induction argument in \cite[Proof of Lemma 3]{banerjeeetal2025} to obtain the desired estimate
\eqref{eq:disc:sum:u0:apr:bd}.

\medskip
\noindent\underline{Proof of Part iii.} The bound \eqref{eq:disc:fz:bd} follows immediately from
parts (i) and (ii). Indeed,
Part i gives $|F(x^\ell)|^2
\le
2A_\infty^4A_0^2B^2N h(\ell)^{2\alpha_0}
+
2A_\infty^4B^2Nd$ while Part ii gives  $h(\ell)^{2\alpha_0}
\le
M^{2\alpha_0}\Bigl(d^{\frac{1}{1-\alpha_0}}+h(0)\Bigr)^{2\alpha_0}
\Bigl((\eta \ell)^{\frac{2\alpha_0}{1-\alpha_0}}\vee 1\Bigr)$, and
combining them leads to \eqref{eq:disc:fz:bd}.

To prove \eqref{eq:disc:jf:z:bd}, write $F_i=(F_{i1},\dots,F_{id})$, and let
$\partial_{j\ell}$ denote differentiation with respect to the $\ell$-th coordinate of $x_j$.
For $u\in\{1,\dots,N\}$, write
\[
a_i:=a(x_i),\qquad a_u:=a(x_u),\qquad U_{0,u}:=U_0(x_u),\qquad
k_{ui}:=k(x_u,x_i).
\]
Then
\[
F_{ik}(x)=\frac1N\sum_{u=1}^N a_i a_u\Bigl((\partial_{x_k}k)(x_u,x_i)-k_{ui}\partial_k U_0(x_u)\Bigr).
\]
Since $a\in W^{1,\infty}$, the weak product rule gives, for a.e.\ $x$,
\begin{equation*}
\partial_{j\ell}F_{ik}(x)=S_1(i,j,k,\ell)+S_2(i,j,k,\ell)+S_3(i,j,k,\ell)+S_4(i,j,k,\ell),
\end{equation*}
where
\begin{align*}
S_1(i,j,k,\ell)
&:=
\frac1N \mathbf 1(i=j)\sum_{u=1}^N
\partial_\ell a(x_i) a(x_u)\Bigl((\partial_{x_k}k)(x_u,x_i)-k_{ui}\partial_kU_0(x_u)\Bigr),\\
S_2(i,j,k,\ell)
&:=
\frac1N a(x_i) \partial_\ell a(x_j)\Bigl((\partial_{x_k}k)(x_j,x_i)-k_{ji}\partial_kU_0(x_j)\Bigr),\\
S_3(i,j,k,\ell)
&:=
\frac1N a(x_i)a(x_j)\Bigl(
(\partial_{x_\ell}\partial_{x_k}k)(x_j,x_i)
-(\partial_{x_\ell}k)(x_j,x_i) \partial_kU_0(x_j)
-k_{ji} \partial_{\ell k}^2U_0(x_j)
\Bigr),\\
S_4(i,j,k,\ell)
&:=
\frac1N \mathbf 1(i=j)\sum_{u=1}^N
a(x_i)a(x_u)\Bigl(
(\partial_{y_\ell}\partial_{x_k}k)(x_u,x_i)
-(\partial_{y_\ell}k)(x_u,x_i) \partial_kU_0(x_u)
\Bigr).
\end{align*}

In particular, we recall that Equation~\eqref{eq:disc:bd:norm:f} shows $F$ is locally bounded.
We also know each $S_r(i,j,k,\ell)$ is locally essentially bounded:
$a$ and its weak derivatives are bounded, the kernel derivatives are
bounded by Assumption~\ref{ass:rkhs}(K2), $\nabla U_0$ is locally
bounded, and $\|\nabla^2U_0\|_{\mathrm{op}}\le C_{U_0}$. Hence
\begin{equation}
\label{eq:disc:f:loc:w1inf}
F\in W_{\mathrm{loc}}^{1,\infty}
\bigl((\mathbb R^d)^N;(\mathbb R^d)^N\bigr).
\end{equation}

Now by $(x_1+x_2+x_3+x_4)^2\le 4(x_1^2+x_2^2+x_3^2+x_4^2)$,
\begin{equation}
\label{eq:anch:jf:hs:split}
\|JF(x)\|_{\mathrm{HS}}^2
=
\sum_{i,j=1}^N\sum_{k,\ell=1}^d |\partial_{j\ell}F_{ik}(x)|^2
\le
4\sum_{r=1}^4 \Sigma_r,
\end{equation}
where
\[
\Sigma_r:=\sum_{i,j=1}^N\sum_{k,\ell=1}^d |S_r(i,j,k,\ell)|^2.
\]

For $S_1$, Jensen's inequality gives
\begin{align}
\Sigma_1
&\le
\frac{2L_a^2A_\infty^2}{N}
\sum_{i,u=1}^N\sum_{k,\ell=1}^d
\Bigl(
|(\partial_{x_k}k)(x_u,x_i)|^2+|k_{ui}|^2 |\partial_kU_0(x_u)|^2
\Bigr)\notag\\
&\le
2L_a^2A_\infty^2B^2Nd^2
+
2L_a^2A_\infty^2B^2A_0^2Nd
\Bigl(\frac1N\sum_{u=1}^N U_0(x_u)\Bigr)^{2\alpha_0}.
\label{eq:disc:bd:sig1}
\end{align}
The last estimate is due to $|\partial_\ell a(x_i)|\le L_a$ and $|a(x_u)|\le A_\infty$ from Assumption~\ref{ass:bl:ext}(A2), $|(\partial_{x_k}k)(x_u,x_i)|\le B$ and $|k_{ui}|\le B$ from Assumption~\ref{ass:rkhs}(K2), and
\[
\frac1N\sum_{u=1}^N |\nabla U_0(x_u)|^2
\le
A_0^2\frac1N\sum_{u=1}^N U_0(x_u)^{2\alpha_0}
\le
A_0^2\Bigl(\frac1N\sum_{u=1}^N U_0(x_u)\Bigr)^{2\alpha_0}
\]
from Assumption~\ref{ass:stab}(S1) together with Jensen's inequality.

Similarly,
\begin{align}
\Sigma_2
&\le
\frac{2L_a^2A_\infty^2}{N}
\sum_{i,j=1}^N\sum_{k,\ell=1}^d
\Bigl(
|(\partial_{x_k}k)(x_j,x_i)|^2+|k_{ji}|^2 |\partial_kU_0(x_j)|^2
\Bigr)\notag\\
&\le
2L_a^2A_\infty^2B^2Nd^2
+
2L_a^2A_\infty^2B^2A_0^2Nd
\Bigl(\frac1N\sum_{j=1}^N U_0(x_j)\Bigr)^{2\alpha_0}.
\label{eq:disc:bd:sig2}
\end{align}

For $S_3$, using $(x+y+z)^2\le 3(x^2+y^2+z^2)$,
\begin{align}
\Sigma_3
&\le
\frac{3A_\infty^4}{N}
\sum_{i,j=1}^N\sum_{k,\ell=1}^d
\Bigl(
|(\partial_{x_\ell}\partial_{x_k}k)(x_j,x_i)|^2
+
|(\partial_{x_\ell}k)(x_j,x_i)|^2 |\partial_kU_0(x_j)|^2
+
|k_{ji}|^2 |\partial_{\ell k}^2U_0(x_j)|^2
\Bigr)\notag\\
&\le
3A_\infty^4B^2Nd^2
+
3A_\infty^4B^2A_0^2Nd
\Bigl(\frac1N\sum_{j=1}^N U_0(x_j)\Bigr)^{2\alpha_0}
+
3A_\infty^4B^2C_{U_0}^2Nd^2.
\label{eq:disc:bd:sig3}
\end{align}

Finally,
\begin{align}
\Sigma_4
&\le
\frac{2A_\infty^4}{N}
\sum_{i,u=1}^N\sum_{k,\ell=1}^d
\Bigl(
|(\partial_{y_\ell}\partial_{x_k}k)(x_u,x_i)|^2
+
|(\partial_{y_\ell}k)(x_u,x_i)|^2 |\partial_kU_0(x_u)|^2
\Bigr)\notag\\
&\le
2A_\infty^4B^2Nd^2
+
2A_\infty^4B^2A_0^2Nd
\Bigl(\frac1N\sum_{u=1}^N U_0(x_u)\Bigr)^{2\alpha_0}.
\label{eq:disc:bd:sig4}
\end{align}

Combining \eqref{eq:anch:jf:hs:split}, \eqref{eq:disc:bd:sig1},
\eqref{eq:disc:bd:sig2}, \eqref{eq:disc:bd:sig3}, and
\eqref{eq:disc:bd:sig4}, and grouping the terms, we obtain
\begin{align}
\|JF(x)\|_{\mathrm{HS}}^2
&\le
4A_0^2B^2\Bigl(4L_a^2A_\infty^2+5A_\infty^4\Bigr)
Nd\Bigl(\frac1N\sum_{i=1}^N U_0(x_i)\Bigr)^{2\alpha_0}
\notag\\
&\qquad
+
4B^2\Bigl(4L_a^2A_\infty^2+5A_\infty^4+3A_\infty^4C_{U_0}^2\Bigr)Nd^2.
\label{eq:disc:bd:jf:apr}
\end{align}
Applying \eqref{eq:disc:sum:u0:apr:bd} from Part ii to
$\frac1N\sum_i U_0(x_i^\ell)$ in \eqref{eq:disc:bd:jf:apr} yields
\eqref{eq:disc:jf:z:bd}. This proves Part iii.

\medskip
\noindent\underline{Proof of Part iv.} Since $C_0\ge \max\{1,D\}$ per \eqref{eq:disc:param:defs}, the choice \eqref{eq:disc:unnor:eta:pre} implies
\[
\eta\le 1
\qquad\text{and}\qquad
\eta^2
\le\frac{\theta^2}{C_0N}
\le\frac{1}{C_0L}
\le\frac{1}{DL}.
\]
Therefore the hypotheses of parts (ii) and (iii) are satisfied.

Since $x^0\in G_K$, we have $h(0)\le K$. Hence Part iii implies that for every $0\le \ell\le L$,
\begin{align}
\label{eq:disc:f:bd:sk}
|F(x^\ell)|^2
&\le
2A_\infty^4A_0^2B^2N \Gamma_K
\Bigl((\eta \ell)^\beta\vee 1\Bigr)
+
2A_\infty^4B^2Nd,\\
\label{eq:disc:jf:bd:sk}
\|JF(x^\ell)\|_{\mathrm{HS}}^2
&\le
4A_0^2B^2C_a Nd \Gamma_K
\Bigl((\eta \ell)^\beta\vee 1\Bigr)
+
4B^2\bigl(C_a+3A_\infty^4C_{U_0}^2\bigr)Nd^2,
\end{align}
for which $\Bigl((\eta \ell)^\beta\vee 1\Bigr)\le \Bigl((\eta L)^\beta\vee 1\Bigr)$. 

We next show that \eqref{eq:disc:unnor:eta:pre} implies
\begin{equation}
\label{eq:disc:eta:small:anch}
\eta^2 N\Bigl((\eta L)^\beta\vee 1\Bigr)\le \frac{\theta^2}{C_0}.
\end{equation}
Indeed, if $(\eta L)^\beta\le 1$, then this is immediate from
\eqref{eq:disc:unnor:eta:pre}. If $(\eta L)^\beta\ge 1$, then $\eta^2N(\eta L)^\beta = N\eta^{2+\beta}L^\beta$. Since $2+\beta=\frac{2}{1-\alpha_0}$, \eqref{eq:disc:unnor:eta:pre} gives
\[
\eta^{2+\beta}L^\beta
\le
\left(\frac{\theta^2}{C_0N^{1-\alpha_0}L^{2\alpha_0}}\right)^{\frac{1}{1-\alpha_0}}
L^{\frac{2\alpha_0}{1-\alpha_0}}
\le
\frac{\theta^2}{C_0N},
\]
because $\theta\in[0,1]$ and $1/(1-\alpha_0)\ge 1$. This proves
\eqref{eq:disc:eta:small:anch}.

Now we use the second part of the definition of $C_0$, namely
\[
C_0
\ge
4\Bigl[
2C_{U_0}A_\infty^4A_0^2B^2\Gamma_K
+
2C_{U_0}A_\infty^4B^2d
+
4A_0^2B^2C_a d \Gamma_K
+
4B^2(C_a+3A_\infty^4C_{U_0}^2)d^2
\Bigr].
\]
Multiplying \eqref{eq:disc:f:bd:sk} by $\eta^2 C_{U_0}$ and using
\eqref{eq:disc:eta:small:anch}, we obtain
\begin{align}
\label{eq:disc:f:small}
\eta^2 C_{U_0}|F(x^\ell)|^2
&\stackrel{\eqref{eq:disc:f:bd:sk}}{\le}
\eta^2N\Bigl[
2C_{U_0}A_\infty^4A_0^2B^2\Gamma_K\Bigl((\eta L)^\beta\vee 1\Bigr)
+
2C_{U_0}A_\infty^4B^2d
\Bigr]\notag\\
&\stackrel{\eqref{eq:disc:c0anc}}{\le}
\frac{C_0}{4}\eta^2N\Bigl((\eta L)^\beta\vee 1\Bigr)
\stackrel{\eqref{eq:disc:eta:small:anch}}{\le}
\frac{\theta^2}{4}.
\end{align}
Similarly, multiplying \eqref{eq:disc:jf:bd:sk} by $\eta^2$ gives
\begin{align}
\label{eq:disc:jf:small}
\eta^2 \|JF(x^\ell)\|_{\mathrm{HS}}^2
&\stackrel{\eqref{eq:disc:jf:bd:sk}}{\le}
\eta^2N\Bigl[
4A_0^2B^2C_a d \Gamma_K\Bigl((\eta L)^\beta\vee 1\Bigr)
+
4B^2\bigl(C_a+3A_\infty^4C_{U_0}^2\bigr)d^2
\Bigr]\notag\\
&\stackrel{\eqref{eq:disc:c0anc}}{\le}
\frac{C_0}{4}\eta^2N\Bigl((\eta L)^\beta\vee 1\Bigr)
\stackrel{\eqref{eq:disc:eta:small:anch}}{\le}
\frac{\theta^2}{4}.
\end{align}


Hence for every $t\in[0,\eta]$,
\begin{equation}
\label{eq:disc:jphi:inv:bd:anch}
\|tJF(x^\ell)\|_{\mathrm{op}}
\le
\eta\|JF(x^\ell)\|_{\mathrm{HS}}
\stackrel{\eqref{eq:disc:jf:small}}{\le}
\frac{\theta}{2}
\le
\frac12.
\end{equation}
Therefore $J\phi_t(x)=I+tJF(x)$ is invertible and
\begin{equation}
\label{eq:disc:jphi:inv:op:bd:anch}
\|(J\phi_t(x))^{-1}\|_{\mathrm{op}}\le 2.
\end{equation}

By \eqref{eq:disc:i2:prime} and the assumption $\|HU_0\|_{\mathrm{op}}\le C_{U_0}$ from Assumption~\ref{ass:stab}(S1),
\begin{align}
\label{eq:disc:psi2:bd:anch}
\psi_2(t)
&=
\int_{(\mathbb R^d)^N}
\bigl\langle F(x), \nabla^2 \widehat U_0(\phi_t(x)) F(x)\bigr\rangle
 g_K(\ell,x) dx\notag\\
&\le
C_{U_0}\int_{(\mathbb R^d)^N}|F(x)|^2 g_K(\ell,x) dx
\stackrel{\eqref{eq:disc:f:small}}{\le}
\frac{\theta^2}{4\eta^2}.
\end{align}
By \eqref{eq:disc:i1:prime} and \eqref{eq:disc:jphi:inv:op:bd:anch},
\begin{align}
\label{eq:disc:psi1:bd:anch}
\psi_1(t)
&\le
\int_{(\mathbb R^d)^N}
\|JF(x)(J\phi_t(x))^{-1}\|_{\mathrm{HS}}^2
 g_K(\ell,x) dx\notag\\
&\le
\sup_{x\in G_K}\|JF(x)\|_{\mathrm{HS}}^2
\sup_{x\in G_K}\|(J\phi_t(x))^{-1}\|_{\mathrm{op}}^2\le
4\sup_{x\in G_K}\|JF(x)\|_{\mathrm{HS}}^2
\le
\frac{\theta^2}{\eta^2}.
\end{align}
Therefore, by \eqref{eq:disc:e2der}, $E''(t)=\psi_1(t)+\psi_2(t)
\le
\frac{5\theta^2}{4\eta^2}$ for $t\in[0,\eta]$ as claimed in \eqref{eq:disc:e2:bd:anch}. Finally, $\int_0^\eta (\eta-t)E''(t) dt
\le
\int_0^\eta(\eta-t)\frac{5\theta^2}{4\eta^2} dt=
\frac{5}{8}\theta^2
\le
\theta^2$
which is \eqref{eq:disc:rem:bd:anch}. This completes the proof. 
\end{proof}

\end{document}